%% file: main.tex
\documentclass[12pt]{article}
\usepackage{fullpage}
\usepackage{amsthm,amsmath}
\usepackage{amssymb}
\usepackage{graphicx}
\usepackage{caption}
\usepackage{subfigure}
\usepackage{wrapfig}
\usepackage{color}
\usepackage{epstopdf}
\usepackage{url}
\usepackage{hyperref}
\graphicspath{{20news/}{sentiment/}{toy/}}
\begin{document}
\include{notation_song}
\title{Estimating Posterior Ratio for Classification: \\ Transfer Learning from Probabilistic Perspective}
\author{ Song Liu\thanks{Institute of Statistical Mathematics, Tokyo, Japan}, \and Kenji Fukumizu\thanks{Institute of Statistical Mathematics, Tokyo, Japan} }
\date{}
\maketitle

\begin{abstract}
	Transfer learning assumes classifiers of similar tasks share certain parameter structures. 
	Unfortunately, modern classifiers uses sophisticated feature representations with huge parameter spaces which lead to costly transfer. 
	Under the impression that changes from one classifier to another should be ``simple'', an efficient transfer learning criteria that only learns the ``differences'' is proposed in this paper.
	We train a \emph{posterior ratio} which turns out to minimizes the upper-bound of the target learning risk.
	The model of posterior ratio does not have to share the same parameter space with the source classifier at all so it can be easily modelled and efficiently trained. The resulting classifier therefore is obtained by simply multiplying the existing probabilistic-classifier with the learned posterior ratio.	
\end{abstract}

\section{Introduction}
%In order to bring unique user experiences, many machine learning algorithms on smart-devices can automatically adapt themselves
%As smart-devices demands more personalized experiences, conventional supervised learning algorithms trained on a general-purpose dataset can only provide a universal classifier that is not fully tuned for individuals. 
%As many machine learning algorithms provide various services through personal smart-devices, the on-board learning algorithms are required to perform quick adjustment for enhanced personal experiences. However, samples collected from users are usually low in quantity and quality. It is unlikely to train a model from scratch using conventional supervised learning techniques.

%supervised learning has to be trained from labelled samples, 
Transfer learning \cite{TransferLearningSurvey,raina2006TransferLearning,regularizedMultitaskLearning} trains a classifier using limited number of samples with the help of abundant samples drawn from another similar distribution.
Specifically, we have a \emph{target task} providing a very small dataset $\dataset_P$ as well as a slightly different \emph{source task} with a large dataset $\dataset_Q$. 
The \emph{Transfer Learning} \cite{TransferLearningSurvey,raina2006TransferLearning,regularizedMultitaskLearning} usually refers to procedures that make use of the similarity between two learning tasks to build a superior classifier using both datasets. 
In this paper, we focus on probabilistic classification problems where the goal is to learn a class posterior $p(y|\boldx)$ over $\dataset_P$, where $p(y|\boldx)$ is the conditional probability of class labels given an input $\boldx$. 

%As smart-device users demands more personalized experiences, Machine Learning algorithms are 
%As many Machine Learning algorithms run on smart-devices, there is a trend suggests supervised learning should also quickly adjust themselves to individual users and offer personalized services, rather than provide a universal classifier which is trained on the general-purpose dataset. 

Due to its complexity of parametrization, the predicting function is usually encoded in the hardware and executed with great efficiency, thus it is reasonable to look at a composite algorithm that consists of two parts: 
a fixed but fast build-in classifier offering complicated predicting pattern and a light-weight procedure works as an adapter that transfers the classifier for a variety of slightly different situations. 
For example, a general-purpose facial recognition built in a camera cannot change its predicting behavior once its model is trained, however the camera may learn transfer models and adjust itself for recognizing a target user. The challenge is, 
the transfer procedure is expected to response rapidly while learning over the entire feature set of the source classifier may slow us down dramatically. 
%the light-weight transfer procedure forbids us using the full feature set given limited on-board computing power. 

%For example, a movie rating classifier predicts whether a user likes a movie or not. It uses external score from a website as the source probabilistic classifier\footnote{ The movie $\boldx$ scores 80\% will be translated as $p(\text{``like''}| \boldx) = 0.8$.}. The transfer algorithm takes the input from a user and tune a target classifier specifically for him. The challenge is, ratings on the website are given by human and our transfer algorithm can never have access to such human knowledge nor the ``features'' they uses to rate a movie.

Intuitively, learning a transfer model does not necessarily need complicated features. 
%Users may prefer to watch certain movies due to simple reasons, such as casting members or genres. The quality of story and cast performance also plays an important role, however, we may safely assume that such modelling has been take care of in the source classifier and remain unchanged for all users. Thus, we may only model ``incrementally'' during the transfer procedure.
Since the task is still facial recognition, we can assume that the changes from one classifier to another are simple and can be described by a trivial (say linear) model with a few key personal features (say hair-style or glasses). The general human facial modelling also plays an important role, however, we may safely assume that such modelling has been taken care of in the source classifier and remain unchanged in the target task. Thus, we can consider the ``incremental model'' only in the transfer procedure.
%For example, a built-in facial recognition in a surveillance camera cannot change its predicting behavior once its model is trained, however cameras may learn a transfer model that takes target input to adjust itself. The challenge is, the light-weight transfer procedure forbids us using costly computer vision features such as SIFT \cite{lowe1999object} given limited on-board computing power. 
%
%Intuitively, learning an incremental model should not require the full set of features. Since the task is still facial recognition, we can assume that the changes from one classifier to another are simple and can be described by a trivial (say linear) model with a few key features (say lighting conditions). The underlying mechanism of facial modeling using edges or curvatures remains unchanged and should \emph{not} be involved at all. 

One of the popular assumptions in transfer learning is to ``reuse'' the model from the source classifier by training a target classifier and limiting the ``distance'' between it and the source classifier model. Regularization has been utilized to enforce the closeness between learned models \cite{regularizedMultitaskLearning}. More complicated structures, such as dependencies between task parameters are also used to construct a good classifier \cite{raina2006TransferLearning}. 
%non simultaneous... {\color{red} do more reviewing}
As most methods require to learn two classifiers of two tasks simultaneously, some works can take already trained classifiers as auxiliary models and learn to reuse their model structures \cite{yang2007cross,chattopadhyay2012multisource,duan2009domain}. 
%Thus it is the softwares's responsibility to adjust the source classifier and make an accurate prediction.

However, reusing the existing model means we need to bring the entire feature set from the source task and include them in the target classifier during transfer learning, even if we know that a vast majority of them does not contribute to the transition from the source to the target classifier. Such an overly expressive model can be harmful given limited samples in $\dataset_P$. Moreover, the hyper-parameters used for constructing features may also be difficult to tune since the cross-validation may be poor on such a small dataset $\dataset_P$. Finally, obtaining those features in some applications may be time-consuming. 

Another natural idea of transfer learning is to ``borrow'' informative samples from the $\dataset_Q$, and get rid of harmful samples. TrAdaBoost \cite{dai2007boosting} follows this exact learning strategy to assign weights to samples from both $\dataset_P$ and $\dataset_Q$. By assigning high weights to samples contributes to the performance in the target task, and penalizing samples that ``misleads'' the classifier, TrAdaBoost reuses the knowledges from both datasets to construct an accurate classifier on the target task. The idea of importance sampling also gives rise to another set of methods learning weights of samples by using \emph{density ratio estimation} \cite{Covariate_Shift,kanamori2009leastsquare,logistictransfer}. Using unlabelled samples from both datasets, an importance weighting function can be learned. By plugging such function into the empirical risk minimization criterion \cite{Vapnik1998}, we can use samples from the $\dataset_Q$ ``as if'' they were samples from $\dataset_P$.
However, such method can not allow ``incremental modelling'' as well, since it learns a full classifier model during the transfer.  

It can be noticed that if one can directly model and learn the ``difference'' between target and source classifier, one may  use only the incremental features which leads to a much more efficient learning criteria. 

The first contribution of this paper is showing that such ``difference learning'' is in fact the learning of a \emph{posterior ratio} which is the ratio between the posteriors from source and target tasks. 
We show learning such posterior ratio is equivalent to minimizing the upper-bound of the classification error of the target task. Second, an efficient convex optimization algorithm is given to learn the parameters of the posterior ratio model and is proved to give consistent estimates under mild assumptions. Finally, the usefulness of this method is validated over various artificial and real-world datasets. 

However, we do not claim that the proposed method has superior performance against all existing works based on extra assumptions, e.g. the smoothness of the predicting function over unlabeled target samples\cite{duan2009domain,chattopadhyay2012multisource}. 
%First, the proposed method focuses on the probabilistic classification setting rather than just label prediction.
%Second, it does not require some assumptions that other transfer learning methods commonly make, such as the smoothness of the predicting function over unlabeled samples. 
%Third, we do not discuss ``multiple-domain transfer'' scenario in this paper. 
The proposed method is simply a novel \emph{probabilistic} framework working on a very small set of assumptions and offers the flexibility of modelling to transfer learning problems. It is fully expendable to various problem settings once new assumptions are made. 

\section{Problem Setting}
\label{sec.prob.form}
Consider two sets of samples drawn independently from two probability distributions $Q$ and $P$ on $\{-1,1\} \times  \mathbb{R}^{d}$: 
\begin{align*}
\dataset_Q& = \left\{\left(\yqj,\xqj \right)\right \}_{j=1}^{n'} \iid Q,\\
\dataset_P &= \left \{\left(\ypi,\xpi\right)\right \}_{i=1}^n \iid P
\end{align*}
$\dataset_Q$ and $\dataset_P$ are source and target dataset respectively.
We denote $p(y|\boldx)$ and $q(y|\boldx)$ as the class posteriors in $P$ and $Q$ respectively. Moreover, $n\ll n'$.

%We also assume $p(y|\boldx)$ and $q(y|\boldx)$ belongs to the exponential family, e.g. $p(y|\boldx)$ can be parametrized as:
%\begin{align}
%\label{eq.pq.parametrized}
%	p(y|\boldx; \boldbeta) = \frac{1}{Z(\boldx;\boldbeta)} \exp\left(y\cdot \sum_{i=1}^{m} \beta_i h_i(\boldx))\right),
%\end{align}
%where $Z(\boldx;\boldbeta) $ is the normalization function:
%\begin{align*}
%Z(\boldx;\boldbeta) = \sum_{y\in \{-1, 1\}} \exp\left(y\cdot \sum_{i=1}^{m} \beta_i h_i(\boldx))\right).
%\end{align*}
%$q(y|\boldx)$ can be parameterized similarly.

Our target is to obtain an estimate of the class posterior $\hat{p}(y|\boldx)$ and predict the class label of an input $\boldx$ by 
$
\hat{y} = \argmax_{y \in \{-1,1\}} \hat{p}(y|\boldx).
$

Clearly, if $n$ is large enough, one may apply logistic regression \cite{logistic_regression,zhu2001kernel} to obtain a good estimate. 
In this paper, we focus on a scenario where $n$ is relatively small and $n'$ is sufficiently large. Thus, it is desirable if we can \emph{transfer} information from the source task to boost the performance of our target classifier.

%Different from previous settings, we focus on a \emph{lazy transfer setting}: We wish to prepare a pre-trained classifier and defer the transferring to a later stage so that the transferring algorithm does not have to handle complex models and large amount of data. As it is mentioned in Section \ref{sec.intro}, such setting may find many applications which require rapid transfer after seeing data for $T_P$. In the next section, we illustrate the detailed methodology.

\section{Composite Modeling}
\label{sec.lazy.trans.modelling}
%At a glance, this problem setting is absurd, since if the target is to learn a posterior $p(y|\boldx)$, we should directly go for it, rather than using a two-step procedure, as it is against the well-know 
%Clearly, one should learn a classifier for $T_P$ directly if enough samples are available, since Vapnik's principle \cite{Vapnik1998} states:
%\begin{quote}
%	When solving a problem of interest, do not solve a more general problem as an intermediate step.
%\end{quote}
%However, we show that in the case of transfer learning, violating this principle may allow us use extra information provided by another dataset $\dataset_Q$. 	
%The key idea lies with posterior decomposition.
Note that the posterior $p(y|\boldx)$ can be decomposed into 
\begin{align*}
p(y|\boldx) = \frac{p(y|\boldx)}{q(y|\boldx)} \cdot q(y|\boldx),
\end{align*}
where $\frac{p(y|\boldx)}{q(y|\boldx)}$ is the \emph{class posterior ratio}, and the $q(y|\boldx)$ is a \emph{source classifier}.

This decomposition leads to a simple transfer learning methodology: Model and learn the posterior ratio and general-purpose classifier \emph{separately}, then later multiply them together as an estimate of the posterior . 

The main interest of this paper is learning such composite model using samples from $\dataset_P$ and $\dataset_Q$. Now, we introduce two parametric models $g(y, \boldx; \boldtheta)$ (or $g_\boldtheta$ for short) and $q(y, \boldx; \boldbeta)$ (or $q_\boldbeta$ for short) for $\frac{p(y|\boldx)}{q(y|\boldx)}$ and $q(y|\boldx)$ respectively. 
%	
%	At a glance, this problem formulation is absurd, since if the target is to learn a posterior $p(y|\boldx)$, we should directly go for it, rather than learning two split components, as it is against the well-know Vapnik's principle \cite{Vapnik1998} in Machine Learning:
%	\begin{quote}
%	When solving a problem of interest, do not solve a more general problem as an intermediate step.
%	\end{quote}
%	However, from now on, we show that in the case of transfer learning, violating this principle by making use of extra information provided by another dataset $\dataset_Q$ makes more sense.
%	
%		We hope to get a good estimate of a target posterior $p(y|\boldx)$, and main concern of this paper is to illustrate how the information contained in a related general-purpose dataset $\dataset_Q$ can help us build a better $p(y|\boldx)$.
%		
\subsection{Kullback-Leibler Divergence Minimization}
\label{sec.trasnfer.bound}
A natural way of learning such a model is to minimize the Kullback-leibler (KL) \cite{Annals-Math-Stat:Kullback+Leibler:1951} divergence between the true posterior and our composite model. 
\begin{mydefinition}[Conditional KL Divergence]
	\begin{align*}
	\mathrm{KL}\left[p\|q\right] = P\log \frac{p(y|\boldx)}{q(y|\boldx)} ,
	\end{align*}
\end{mydefinition}
We denote $P f$ as the short hand of the integral/sum of a function $f$ over a probability distribution $P$ on its domain.

Now, we proceed to obtain the following upper-bound of KL divergence from $p$ to the composite model:
%	
%	Before proceed, we introduce a measure of discrepancy between general-purpose and target learning tasks. When designing algorithm, we should also take such information into account.
%	\begin{assum}[Bounded Discrepancy]
%	\begin{align*}
%	\gamma_\text{max} \ge \sup_{x,y} \left|\frac{p(y,x)}{q(y,x)} -1 \right|
%	\end{align*}
%	\end{assum}
%	We define the ``dissimilarity'' between two tasks by using the maximum deviation of the ratio between two density functions. Clearly, if two tasks are the same, $\gamma_\text{max} = 0$.

%	Now, we introduce $g(y,\boldx;\boldtheta)$ and $h(y,\boldx;\boldbeta)$ as the models of the posterior ratio $\frac{p(y|\boldx)}{q(y|\boldx)}$ and the general-purpose classifier $q(y|\boldx)$ respectively. The \emph{transfer inequality} refers to:
\begin{prop}[Transfer Learning Upper-bound]
	if $
	%	\label{eq.bounded.ratio}
		\frac{p(y,\boldx)}{q(y,\boldx)}\le C_\mathrm{max} <\infty
$
	and $0<q_\boldbeta < 1$, then the following inequality holds
	\begin{align}
	\label{eq.TI}
	\mathrm{KL}\left[p\|g_\boldtheta \cdot q_\boldbeta \right] \le \mathrm{KL}\left[p\|g_\boldtheta q\right] + C_\mathrm{max}\mathrm{KL}\left[q\|q_\boldbeta\right] + C', 
	\end{align}
	where $C'$ is a constant that is irrelevant to $\boldtheta$ or $\boldbeta$. 
\end{prop}
\begin{proof}
	\begin{align}
	&\text{KL}\left[p\|g_\boldtheta \cdot h_\boldbeta\right] = \text{KL}\left[p \| g_\boldtheta \cdot q \cdot \frac{q_\boldbeta}{q}\right] \notag\\
	=& \text{KL}\left[p\|g_\boldtheta \cdot q \right]  - P \log q_\boldbeta + P \log q \notag\\
	=& \text{KL}\left[p\|g_\boldtheta \cdot q \right] -  \int q(y,\boldx) \frac{p(y,\boldx)}{q(y,\boldx)}\log q_\boldbeta \;dy\boldx + P \log q \notag\\
	\le& \text{KL}\left[p\|g_\boldtheta \cdot q \right] + C_\mathrm{max} Q\log q-  C_\mathrm{max} Q\log q_\boldbeta + C' \label{eq.TI.1}\\
	=& \text{KL}\left[p\|g_\boldtheta \cdot q \right] + C_\mathrm{max}\text{KL}\left[q\|q_\boldbeta\right] + C' , \notag
	\end{align}
	where $C' = P\log q - C_\mathrm{max} Q\log q $.
	Further, 
	\begin{align}
	&\text{KL}\left[p\|g_\boldtheta \cdot q \right] + C_\mathrm{max} \text{KL}\left[q\|q_\boldbeta\right] \notag\\
	\approx& - \frac{1}{n} \sum_{i=1}^{n} \log g\left(\ypi,\xpi;\boldtheta\right) \notag\\
	&- C_\mathrm{max} \frac{1}{n'} \sum_{j=1}^{n'} \log q\left(\yqj,\xqj;\boldbeta\right) + C''\label{eq.TI.app} 
	\end{align}
	where $C''$ is a constant that is irrelevant to $\boldtheta$ or $\boldbeta$.
\end{proof}
We may minimize the empirical upper-bound \eqref{eq.TI.app} of KL divergence in order to obtain estimates of $\boldtheta$ and $\boldbeta$. $C_\mathrm{max}$ is an unknown constant introduced in \eqref{eq.TI.1} that illustrates the how dissimilar these two tasks are.  Such upper-bound in \eqref{eq.TI} formalizes the common intuition that ``if two tasks are similar, transfer learning should be easy, '' since the more similar two tasks are, the smaller the $C_\mathrm{max}$ is, and the tighter the bound is.

Note that the minimizing \eqref{eq.TI.app} leads to two \emph{separate} maximum likelihood estimation (MLE). The MLE of the second likelihood term of bound \eqref{eq.TI.app}
\begin{align*}
\hat{\boldbeta} = \argmax_{\boldbeta}\frac{1}{n'} \sum_{j=1}^{n'} \log q\left(\yqj,\xqj;\boldbeta\right)
\end{align*}
leads to a conventional MLE of a posterior model, and has been well studied. $q$ can be efficiently modeled and trained using techniques such as logistic regression \cite{logistic_regression,zhu2001kernel}. Here we consider it is already given.  However, maximizing the first likelihood term, a posterior ratio 
\begin{align}
\label{eq.ratio.optimization}
\hat{\boldtheta} = \argmax_{\boldtheta}\frac{1}{n} \sum_{i=1}^{n} \log g\left(\ypi,\xpi;\boldtheta\right)
\end{align}
is our main focus. In the next section, we show the modelling and learning of the posterior ratio is feasible and computationally efficient. 
\subsection{Posterior Ratio Model}
\label{sec.model}
Although it is not necessary, to illustrate the idea behind the posterior ratio modelling, we assume $p(y|\boldx)$ and $q(y|\boldx)$ belongs to the exponential family, e.g. $p(y|\boldx)$ can be parametrized as:
\begin{align}
\label{eq.pq.parametrized}
	p(y|\boldx; \boldbeta) \propto  \exp\left(y\cdot \sum_{i=1}^{m} \beta_i h_i(\boldx))\right),
\end{align}
%Now we show that directly modeling the posterior ratio leads to great simplification under assumptions that a vast majority of features in both posterior models does not contribute to the change from $p$ to $q$.
Given the parametrization model \eqref{eq.pq.parametrized}, consider the ratio between $p$ and $q$:
\begin{align*}
	\frac{p(y|\boldx;\beta_p)}{q(y|\boldx;\beta_q)} \propto \exp\left(y\sum_{i=1}^m (\beta_{p,i}-\beta_{q,i}) h_i(\boldx)\right).
\end{align*}
For all $\beta_{p,i}-\beta_{q,i} = 0$, factor feature $f_i $ is nullified, and therefore can be ignored when modelling the ratio. In fact, once the ratio is considered, the separate $\boldbeta_{p}$ and $\boldbeta_{q}$ does not have to be learned, but only their difference $\theta_i = \beta_{p,i}-\beta_{q,i}$ is sufficient to describe the transition from $p$ to $q$. Thus, we write our posterior ratio model as 
\begin{align}
\label{eq.model}
	r(y,\boldx; \boldtheta) = \frac{1}{N(\boldx;\boldtheta)} \exp \left( y \sum_{i \in S} \theta_i h_i(\boldx)\right), 
\end{align}
where $S = \{i | \beta_{p, i} - \beta_{q, i} \neq 0\}$ and $N(\boldx;\boldtheta)$ is the normalization term defined as
\begin{align*}
N(\boldx;\boldtheta) = \sum_{y\in \{-1,1\}} q(y|\boldx) \exp\left( y \sum_{i \in S} \theta_i h_i(\boldx) \right).
\end{align*}
Such normalization is due to the fact that we are minimizing the KL divergence between $p(y|\boldx)$ and $g(y,\boldx;\boldtheta) q(y|\boldx)$, we need to make sure that $g(y,\boldx;\boldtheta) q(y|\boldx)$ is a valid conditional probability, i.e., :
$\sum_y q(y|\boldx) ~g_\boldtheta(y,\boldx;\boldtheta) = 1.$ 

%If one can assume that many factors does not ...
%This modelling technique allows us to consider only the ``differences'' between two posteriors, and assumptions can be made regarding to this model:
This modelling technique gives us great flexibilities since it only concerns the ``effective features'' $\{h_i\}_{i \in S}$ rather than the entire feature set $\{h_1, h_2, \dots, h_m\}$. In this paper, we assume the transfer should be simple, thus the potential feature set only contains ``simple features'', such as linear ones: $h_i(\boldx) = x_i, i\in S$. 
%For example, a facial recognition transfer does not have to deal with any facial modelling feature if the cause of change is simple (e.g. lighting conditions). 

From now on, we simplify $ y \sum_{i \in S} \theta_i h_i(\boldx)$ using a linear representation $\boldtheta^\top \boldf(y,\boldx)$, where \[\boldf(y,\boldx) = [yh_{a_1}(\boldx), yh_{a_2}(\boldx), \dots, yh_{a_{m'}}(\boldx)],\] where $a_1,a_2,\dots, a_{m'} \in S$.
%\begin{itemize}
%	\item Assume change is simple and straightforward (say, $f_i(\boldx) = x_i$) though individual models of classifiers can be very complicated. Since factors shared by both source and target tasks, no matter how sophisticated, are canceled after taking the division. 
%	\item Assume one have the prior knowledge on features participated in the transition from $p$ to $q$, then one can use these features to construct a ratio model even if he does not know any information on the original feature set used by the source classifier. 
%\end{itemize}

%For example, a self-normalizing log-linear model, should be written as the following form:
%\begin{align}
%\label{eq.model}
%g(y,\boldx;\boldtheta) := \frac{\exp\left( \boldtheta^\top \boldf(y,\boldx)\right)}{N(\boldx;\boldtheta)},
%\end{align}
%where $N(\boldx;\boldtheta)$ is defined as
%\begin{align*}
%N(\boldx;\boldtheta) = \sum_{y\in \{1\dots k \}} q(y|\boldx) \exp\left( \boldtheta^\top \boldf(y, \boldx)\right) 
%\end{align*}

However, this modelling also causes a problem: We cannot directly evaluate the output value of this model, since we do not have access to the true posterior $q(y|\boldx)$. Therefore, we can only use samples from $\dataset_Q$ to approximate the normalization term.

\section{Estimating Posterior Ratio}
\label{sec.classposterior}
Now we introduce the estimator of the class-posterior ratio $p(y|\boldx)/{q(y|\boldx)}$. 
 Let us substitute the model of \eqref{eq.model} into the objective \eqref{eq.ratio.optimization}:
\begin{align*}
\hat{\boldtheta}=\argmax_{\boldtheta} &\frac{1}{n}\sum_{i=1}^{n} \log g\left(\ypi, \xpi ;\boldtheta\right) \\
=& \frac{1}{n}\sum_{i=1}^{n} \boldtheta^\top \boldf\left(\ypi, \xpi\right) - \frac{1}{n}\sum_{i=1}^{n}\log N\left(\boldtheta, \xpi\right).
\end{align*}
The normalization term needs to be evaluated in a pointwise fashion $N\left(\boldtheta, \boldx_p^{(i)}\right), \forall \boldx_p^{(i)} \in \dataset_P$. Note that if we have sufficient observations $(y_q, \boldx)$ paired with each $\boldx_p^{(i)}$, i.e. $\left\{ \left( \yqj, \boldx\right)\right\}_{j=1}^k \sim Q, \boldx = \boldx_p^{(i)}$, such normalization can be approximated efficiently via \emph{sample average}:
\[
N(\boldtheta, \boldx) \approx \frac{1}{k}\sum_{j=1}^{k} \exp\left(\boldtheta^\top \boldf\left(\yqj, \boldx\right)\right).
\]
\begin{figure}[t]
	\begin{center}
		\includegraphics[width=0.4\textwidth]{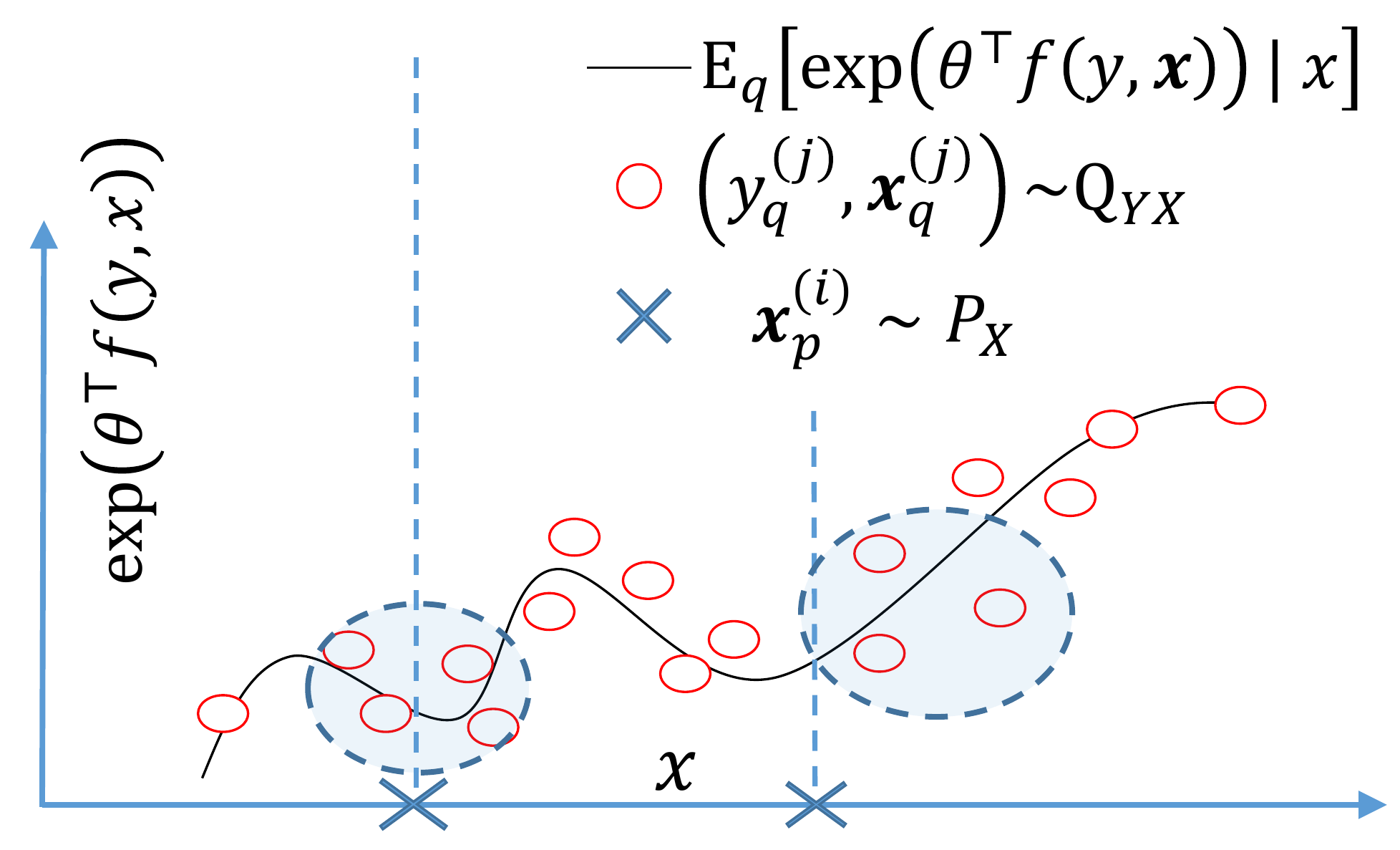}
	\end{center}
	\caption{Approximate $N(\boldtheta,\boldx)$ using nearest neighbours.}
	\label{fig.illus.nei}
\end{figure}
However, in practice not many observed samples may be paired with $\boldx_p^{(i)}$. Especially when $\boldx$ is in a continuous domain, we may not observe any paired sample at all. We may consider using the neighbouring pairs $\left( \yqj, \boldx_q^{(j)} \right)$ where $\boldx_q^{(j)}$ is a neighbour of $\boldx_p^{(i)}$ to approximate $N(\boldtheta, \boldx_p^{(i)})$, which naturally leads to the idea of $k$-nearest neighbours ($k$-NN) estimation of such quantity (see Figure \ref{fig.illus.nei}):
\begin{align*}
N(\boldtheta, \boldx_p^{(i)}) &\approx \widehat{N}_{n',k}\left(\boldtheta; \boldx_p^{(i)}\right) \\
&=\frac{1}{k}\sum_{j\in \mathcal{N}_{n'} \left(\boldx_p^{(i)}, k\right)} \exp\left(\boldtheta^\top\boldf\left(\yqj, \boldx_q^{(j)}\right)\right),	
\end{align*}
where 
\[
\mathcal{N}_{n'}(\boldx^{(i)}_p, k) = \left\{ j \biggl| \boldx^{(j)}_q \text{is one of the }k\text{-NNs of } \boldx^{(i)}_p\right\}.
\]
Now we have a ``computable'' approximation to the posterior ratio model:
\begin{align*}
	 g_{n'}(y,\boldx;\boldtheta) = \frac{\exp\left( \boldtheta^\top \boldf(y, \boldx)\right)}{N_{n',k}(\boldx;\boldtheta)}.
\end{align*}
The resulting optimization is 
\begin{align}
\label{eq.obj}
\hat{\boldtheta} = \argmin_{\boldtheta} ~~ &\ell(\boldtheta; \dataset_P, \dataset_Q) = -\frac{1}{n}\sum_{i=1}^{n} \boldtheta^\top \boldf\left(\ypi, \boldx_p^{(i)} \right) \notag\\
+ \frac{1}{n}\sum_{i=1}^{n}\log \frac{1}{k}&\sum_{j\in \mathcal{N}_q \left(\boldx_p^{(i)}, k\right)} \exp\left(\boldtheta^\top\boldf\left(\yqj, \boldx_q^{(j)}\right)\right),
\end{align}
which is convex. Note $\ell$ represents the \emph{negative} likelihood.

%Moreover, in this work, we assume that posteriors are \emph{similar}, so they may share many feature functions. The resulting posterior ratio may be \emph{sparse} in classification features. Therefore, in this paper, we use a sparsity inducing norm together with the main objective function to obtain a sparse solution:
Moreover, if we assume that the changes between two posteriors are ``mild'', i.e. $\|\boldtheta_{i}\| = \|\boldbeta_{p,i} - \boldbeta_{q,i}\|$ is small, we may use an extra $\ell_2$ regularization to restrict the magnitude of our model parameter $\boldtheta$:
\begin{align}
\label{eq.main.obj}
\argmin_{\boldtheta} \ell(\boldtheta) + \lambda \|\boldtheta\|_2,
\end{align}
where the $\lambda$ is a regularization term and can be chosen via likelihood cross-validation in practice.	Finally the gradient of $\ell(\boldtheta)$ is given as
\begin{align*}
	\nabla_\boldtheta \ell(\boldtheta) = &-\frac{1}{n}\sum_{i=1}^{n} \boldf\left(\ypi, \boldx_p^{(i)} \right) \\
	&+ \frac{1}{n}\sum_{i=1}^{n} \mathbb{E}_{n'}\left[g_{n'}(y,\boldx;\boldtheta)\boldf(y,\boldx) \biggl| \boldx = \boldx^{(i)}\right],
\end{align*}
where $\mathbb{E}_{n'}\left[Z | \boldx = \boldx^{(i)}\right]$ is the empirical $k$-NN estimate of a conditional expectation over $Q$: 
\begin{align*}
	\mathbb{E}_{n'}\left[Z \biggl| \boldx = \boldx^{(i)}\right] =\frac{1}{k}\sum_{j\in \mathcal{N}_{n'} \left(\boldx^{(i)}, k\right)} Z_j.
\end{align*}
The computation of this gradient is straightforward, and thus we can use any gradient-based method such as quasi-newton to solve the unconstrained convex optimization in \eqref{eq.main.obj}.

It can be noticed that such algorithm is similar to the density ratio estimation method, KLIEP \cite{Covariate_Shift_jour}. Indeed, they are all estimators of learning a ratio function between two probabilities based on maximum-likelihood criteria. However, the proposed method is different from \cite{Covariate_Shift_jour} in terms of modelling, motivation and usage. 

\section{Consistency of the Estimator}
\label{sec.proof}
In this section, we analyze the consistency of the estimator given in \eqref{eq.obj}, i.e. whether the estimated parameter converges to the solution of the population objective function. This result is \emph{not} straightforward since we used an extra $k$-NN approximation in our model so that the model itself is an ``estimate''. The question is, does this approximation lead to a consistent estimator? 

First, we define the estimated and true parameter as:
\begin{align*}
\hat{\boldtheta} &= \argmax_{\boldtheta} \ell(\hat{\boldtheta}; \dataset_P, \dataset_Q)  = P_n~~ \log g_{n'}(y,\boldx; \boldtheta)\\ \boldtheta^* &= \argmax_{\boldtheta} P~~ \log g(y,\boldx; \boldtheta),
\end{align*}
where $P_n$ is the empirical measure of distribution $P$.
\begin{assum}[Bounded Ratio Model]
	\label{ass.bounded.model}
	There exists $1<M_\mathrm{max}<\infty$, so that $\left|\boldtheta^\top \boldf(y,\boldx)\right| \le \log(M_\mathrm{max})$. Moreover, $\boldtheta$ is in a totally bounded metric space and $\max_{y, \boldx}\|\boldf(y, \boldx)\|_2 \le F_\mathrm{max}$ where $0<F_\mathrm{max}<\infty$.
\end{assum}
Therefore $\exp\left(\boldtheta^\top \boldf(y, \boldx)\right), N(\boldx;\boldtheta)$ and $N_{n',k}(\boldx;\boldtheta) \in \left[\frac{1}{M_\text{max}}, M_\text{max} \right]$, and the posterior ratio model is always bounded by constants. It is a reasonable assumption as the posterior ratio measures the ``differences'' between two tasks, the true posterior ratio must be close to one if two tasks are similar.  
\begin{assum}[Bounded Covariate Shift]
	$\frac{p(\boldx)}{q(\boldx)} \le R_\mathrm{max}$.
\end{assum}
The support between $P$ and $Q$ must overlap. If samples in $\dataset_P$ distribute completely differently from those in $\dataset_Q$, it does not make sense to expect the transfer learning method would work well.  
\begin{assum}[Identifiability]
	\label{assum.identi}
	$\boldtheta^*$ is the unique global maximizer of the population objective function $P \log g(y,\boldx;\boldtheta)$, i.e. for all $\epsilon >0$,
	\[
	\sup_{\boldtheta, \|\boldtheta-\boldtheta^*\|\ge \epsilon} P \log g(y,\boldx;\boldtheta) < P \log g(y,\boldx;\boldtheta^*).
	\]
\end{assum}
Then we have the following theorem that states our posterior ratio estimator is \emph{consistent}. 
\begin{mytheorem}
	\label{thm.my.only.theorem}
	Suppose for each $\boldx$, the random variable $\|X-\boldx\|$ is absolutely continuous. If $n\rightarrow \infty$, $n'\rightarrow \infty$, $k_{n'} / \log n' \rightarrow \infty $ and $k_{n'} / n' \rightarrow 0 $, where $k_{n'}$ is the sample dependent version of $k$, the number of nearest neighbors used in $k$-NN approximation. Then under above assumptions, $\hat{\boldtheta} \overset{p}{\to} \boldtheta^*$. Further $	\ell(\hat{\boldtheta}; \dataset_P, \dataset_Q) \overset{p}{\to} \mathrm{KL}\left[p\|q\right]$.
\end{mytheorem}
The proof relies on the following lemma:
\begin{mylemma}
	\label{lem.convergence}
	Under all assumptions stated above, if $n\rightarrow \infty$, $n'\rightarrow \infty$, $k_{n'} / \log n' \rightarrow \infty $ and $k_{n'} / n' \rightarrow 0 $. Then $\sup_\boldtheta \left|P_n~~ \log g_{n'}(y,\boldx; \boldtheta) - P~~ \log g(y,\boldx;\boldtheta) \right| \overset{p}{\to} 0$, i.e. the error caused by approximating objective using samples converges to 0 in probability \textbf{uniformly} w.r.t. $\boldtheta$.
\end{mylemma}
One of the key steps is to decompose the above empirical approximation error of the objective function into: Approximation error caused by using samples from $P$ + Modelling error caused by $k$-NN using samples from $Q$. It can be observed that the bound of density ratio $R_\mathrm{max}$ also contributes to the error. The complete proof is included in the appendix.
\section{Decomposing Paramter vs. Decomposing Model}
\label{sec.decom}
Instead of decomposing the model $p_{\boldtheta, \boldbeta} = g_{\boldtheta} h_{\boldbeta}$ as we propose in this paper, the Model-reuse methods (e.g. \cite{regularizedMultitaskLearning,raina2006TransferLearning}) decompose the \emph{parameter}: $\boldbeta_p = \boldtheta + \boldbeta_q,$
which leads to a problem of minimizing a KL divergence
\begin{align*}
	\min_{\boldbeta_q, \boldtheta} \mathrm{KL}\left[p\|h(\boldtheta+\boldbeta_q)\right].
\end{align*}
Two issues come with this criteria. First, this problem is not identifiable since there exist infinitely many possible combinations of $\boldtheta$ and $\boldbeta_q$ that minimizes the objective function. One must use extra assumptions. Model-reuse methods add a ``regularizer'' on parameter $\boldbeta_q$ using KL-divergence. 
\begin{align}
\label{eq.joint}
\hat{\boldbeta}_q, \hat{\boldtheta} = \argmin_{\boldbeta_q, \boldtheta} \mathrm{KL}\left[p\|h(\boldtheta+\boldbeta_q)\right] + \gamma \mathrm{KL}\left[q\|h(\boldbeta_q)\right],
\end{align}
which implies that the minimizer $\hat{\boldbeta}_q$ should also make the difference between $q$ and $h(\hat{\boldbeta}_q)$ small, in terms of KL divergence, and $\gamma$ is a ``balancing parameter'' has to be tuned using cross-validation which may be poor when the number of samples from $\dataset_P$ is low. As we will show later in the experiments, the choice of $\gamma$ is crucial to the performance when $n$ is small. 

Second, since the model must be normalized, i.e. $\int h(\boldtheta+\boldbeta_q) \; \mathrm{d}y = 1$, so $\boldbeta_q$ and $\boldtheta$ are always coupled, one must always solve them together, meaning the algorithm have to handle the complicated feature space for $\boldbeta_q$ and $\boldtheta$.

However, things are much easier if we have access to the true parameter of the posterior $\boldbeta^*_q$, then we can model the posterior of $p$ as $g(y,\boldx;\boldtheta) q(y|\boldx;\boldbeta^*_q)$, where $g$ is the model of the ratio. This setting leads to the proposed posterior ratio learning method:
\begin{align*}
\hat{\boldtheta} = \argmin_{\boldtheta} \mathrm{KL}\left[p\|g(y,\boldx;\boldtheta) q(y|\boldx;\boldbeta^*_q)\right].
\end{align*}
where $\boldbeta_q^*$ is a constant, so this optimization is with respect to $\boldtheta$ only. 
This paper presents an algorithm that can obtain an estimate of $g(y,\boldx;\boldtheta)$ even if one does not know 
$q(y|\boldx;\boldbeta^*_q)$ exactly. $q(y|\boldx;\hat{\boldbeta}_q)$ is learned separately and is multiplied with $g(y,\boldx;\hat{\boldtheta})$ in order to provide an posterior output. In comparison, the decomposition of model results two independent optimizations and we are free from the join objective where the choice of the parameter $\gamma$ is problematic. Neither do we have to assume that $\boldtheta$ and $\boldbeta_q$ are in the same parameter space.

\begin{figure*}[t]
	\centering
	\subfigure[$\ell(\hat{\boldtheta}; \dataset_P, \dataset_Q)$]{
		\label{fig.conv}
		\includegraphics[width=.3\textwidth]{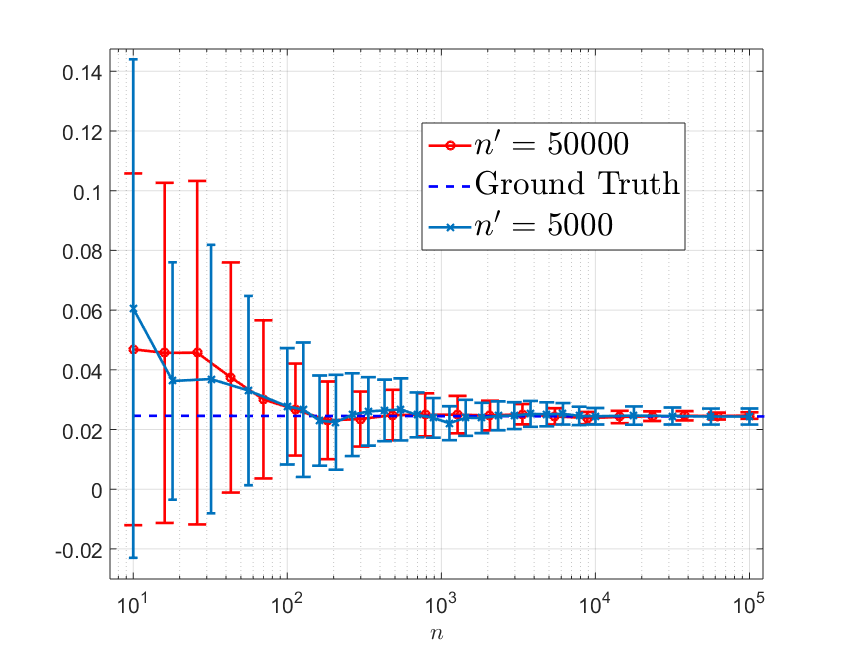}
	}
	\subfigure[ \emph{negative} hold-out likelihood ]{
		\label{fig.ll}
		\includegraphics[width=.3\textwidth]{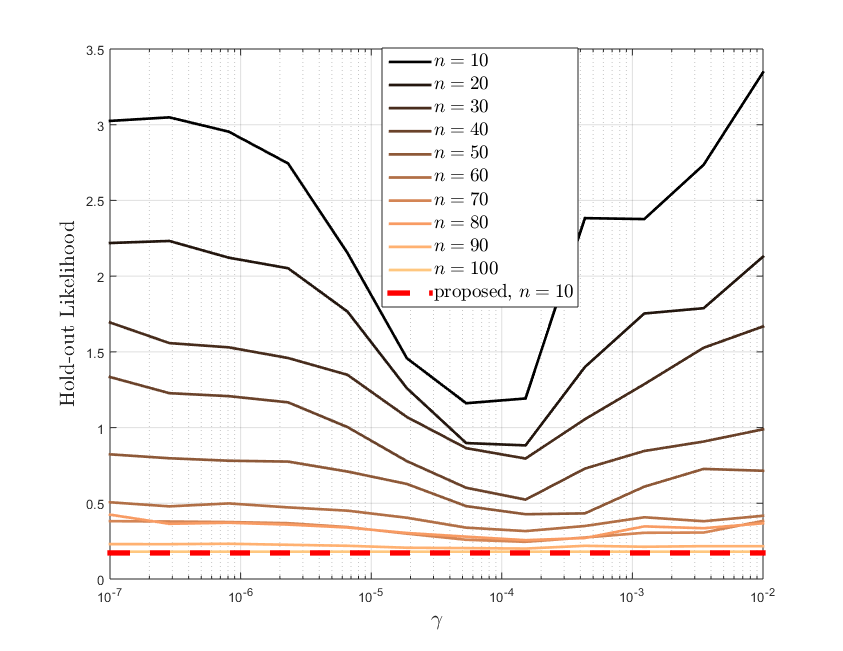}
	}
	\subfigure[Illustration of 4-Gaussian dataset shift]{
		\label{fig.illus}
		\includegraphics[width=.32\textwidth]{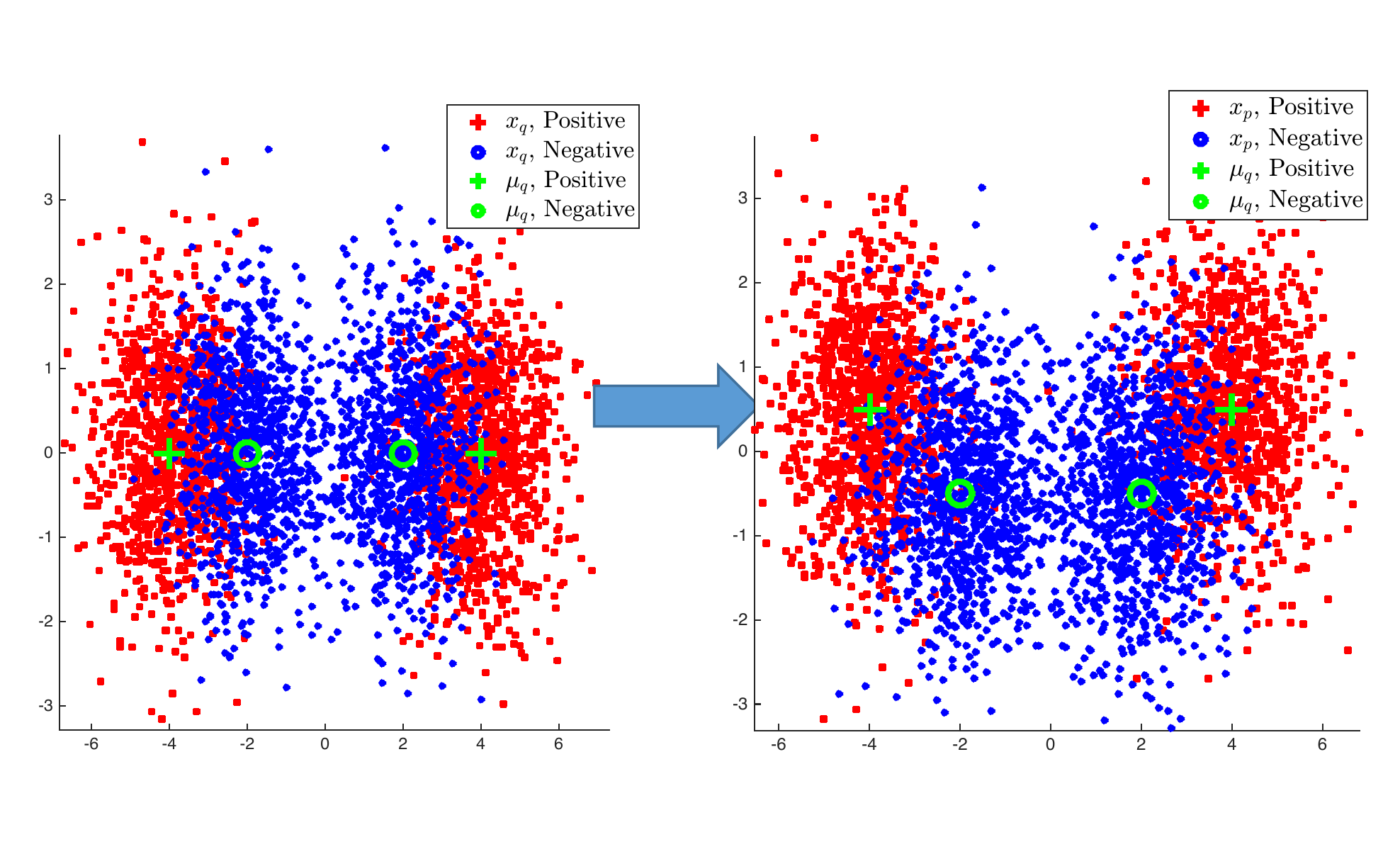}
	}	
	\subfigure[$p(y|\boldx;\hat{\boldbeta}_p)$, miss-rate: 13.8\%.]{
		\label{fig.4gau.p}
		\includegraphics[width=.3\textwidth]{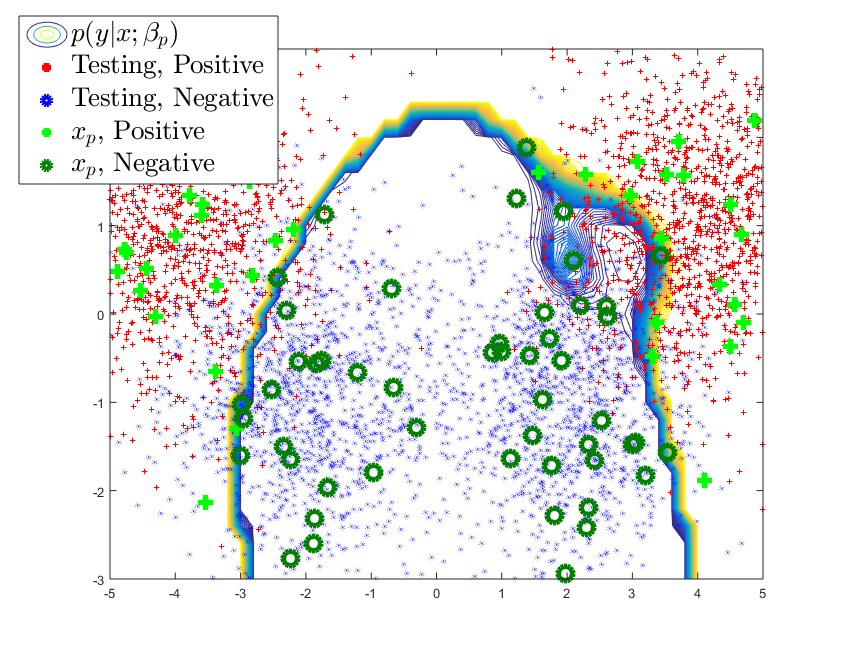}
	}
	\subfigure[$q(y|\boldx;\hat{\boldbeta}_q)$, miss-rate: 15.2\%]{
		\label{fig.4gau.q}
		\includegraphics[width=.3\textwidth]{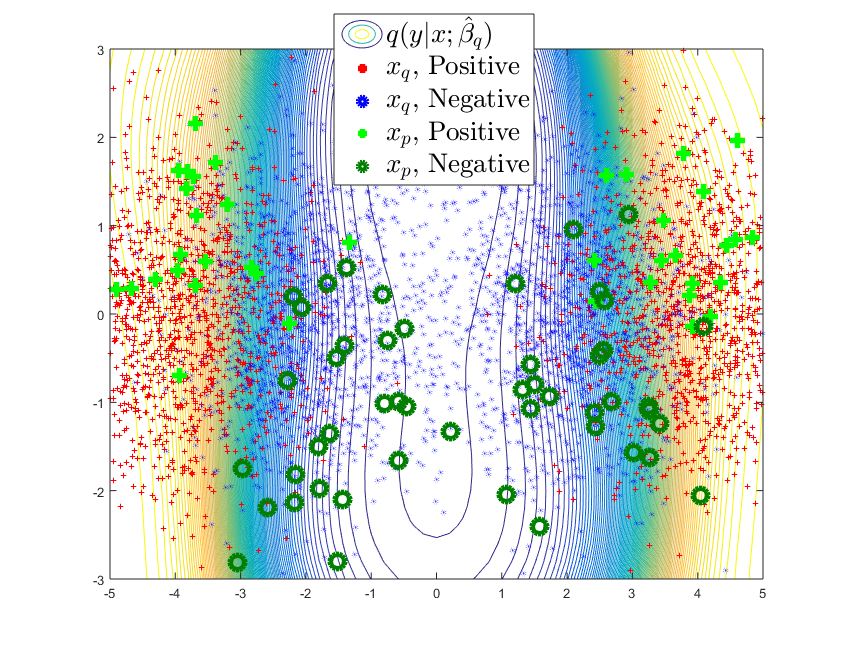}
	}	
	\subfigure[$g(y, \boldx;\hat{\boldtheta} )q(y|\boldx;\hat{\boldbeta}_q)$, miss-rate: 8.0\%]{
		\label{fig.4gau.proposed}
		\includegraphics[width=.3\textwidth]{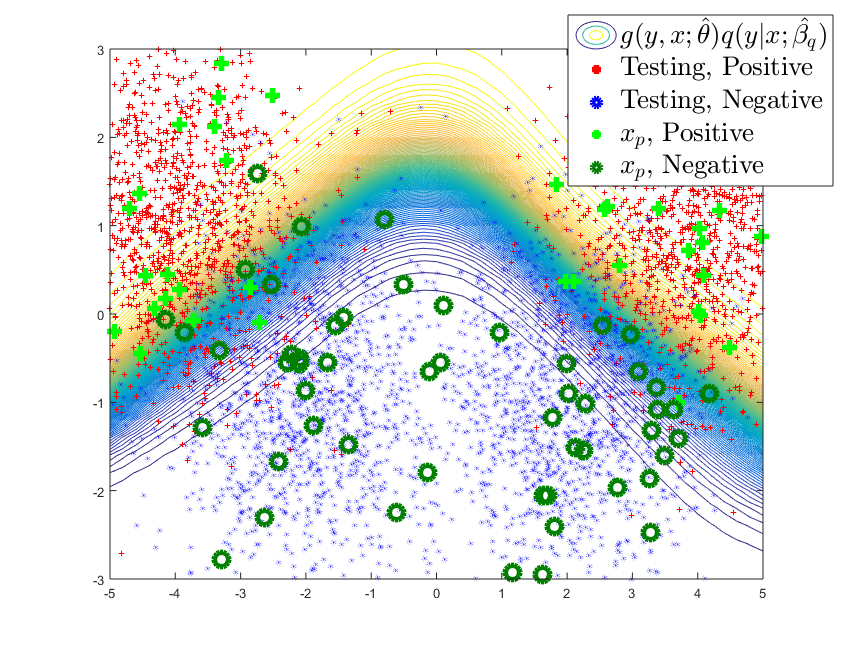}
	}
	\caption{Experiments on artificial datasets}
\end{figure*}
\begin{figure*}[t]
	\centering
	\subfigure[sci.crypt]{\label{fig.20news.12}\includegraphics[width=.3\textwidth]{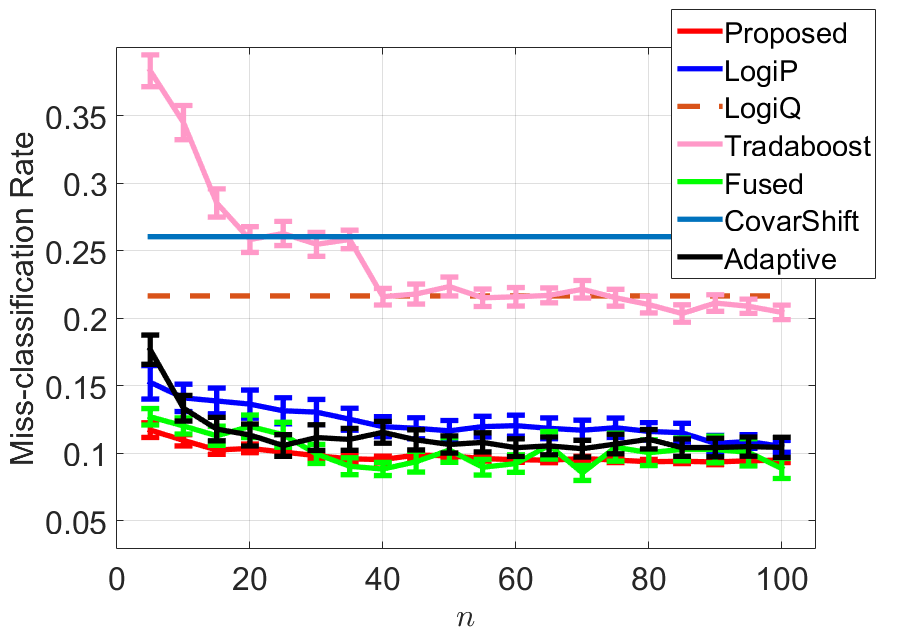}}
	\subfigure[sci.electronics]{\label{fig.20news.13}\includegraphics[width=.3\textwidth]{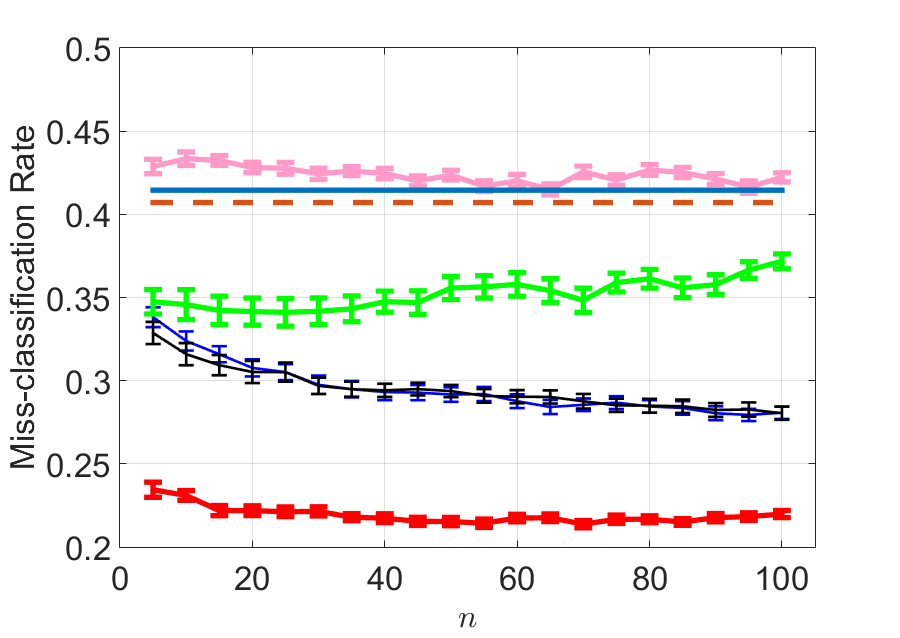}}
	\subfigure[sci.med]{\label{fig.20news.14}\includegraphics[width=.3\textwidth]{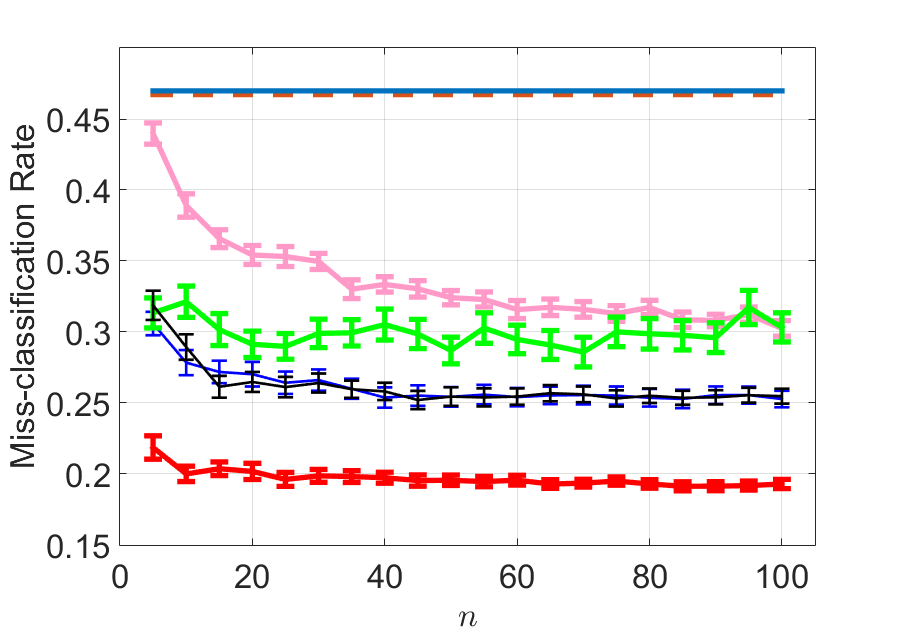}}
	\subfigure[sci.space]{\label{fig.20news.15}\includegraphics[width=.3\textwidth]{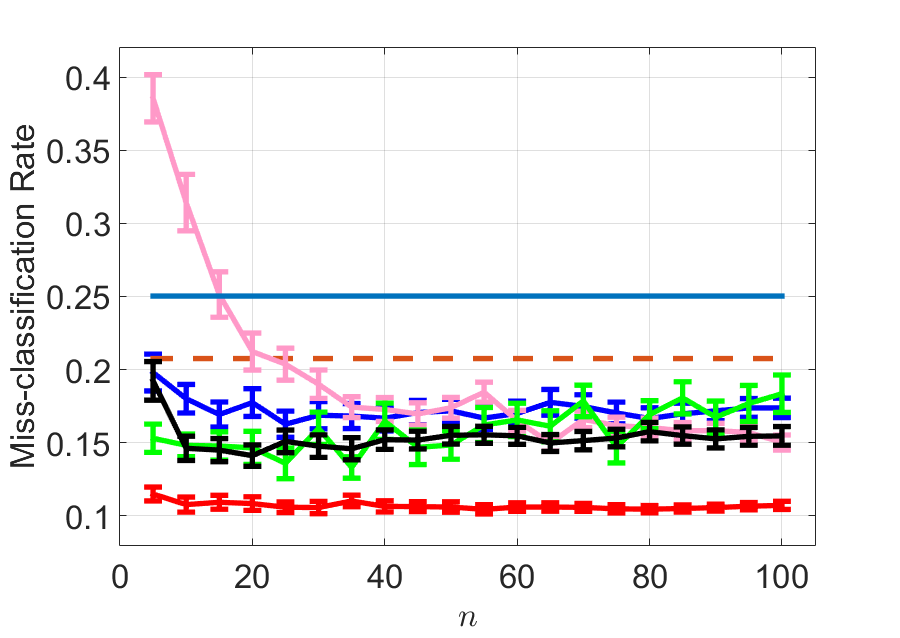}}
	\subfigure[talk.politics.guns]{\includegraphics[width=.3\textwidth]{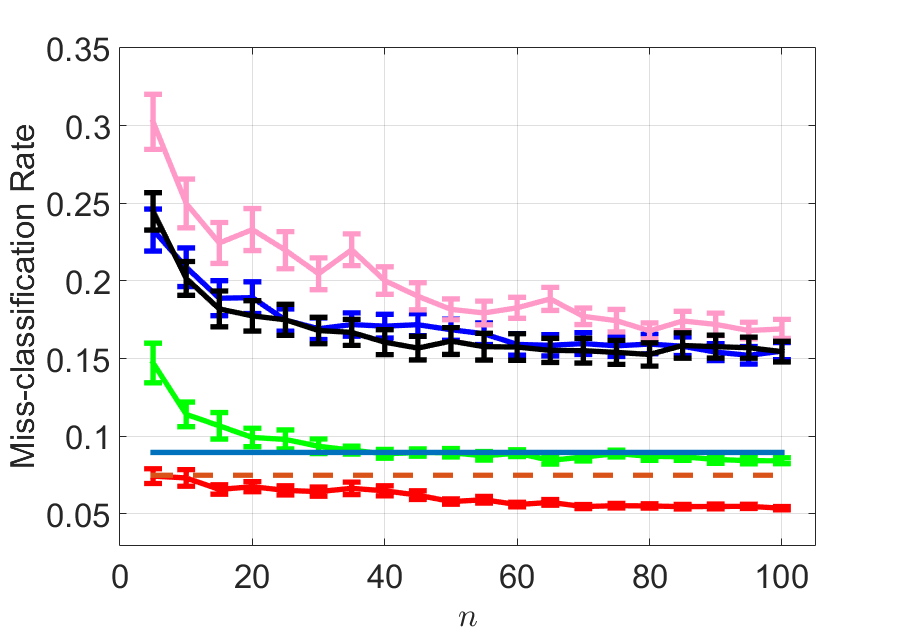}}
	\subfigure[talk.politics.mideast]{\includegraphics[width=.3\textwidth]{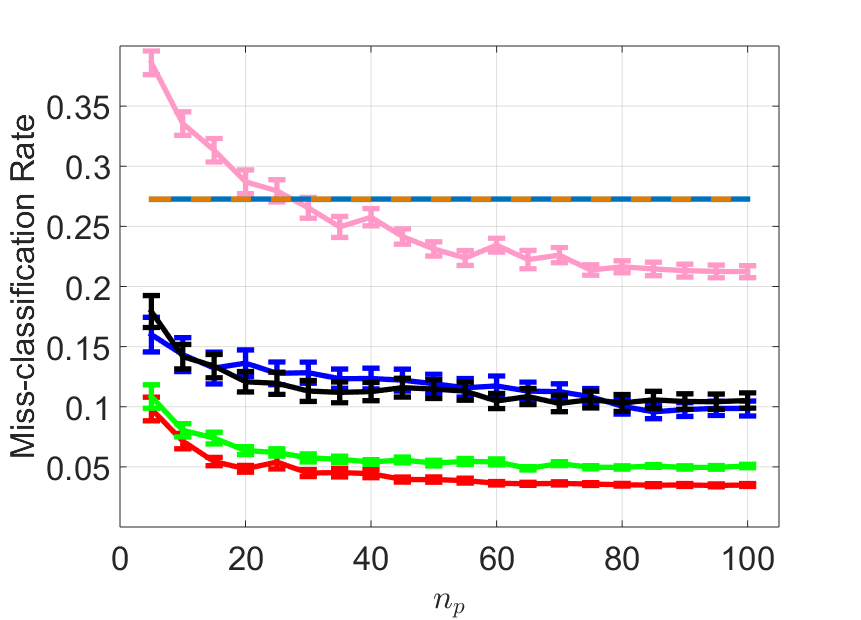}}
	\subfigure[talk.politics.misc]{\includegraphics[width=.3\textwidth]{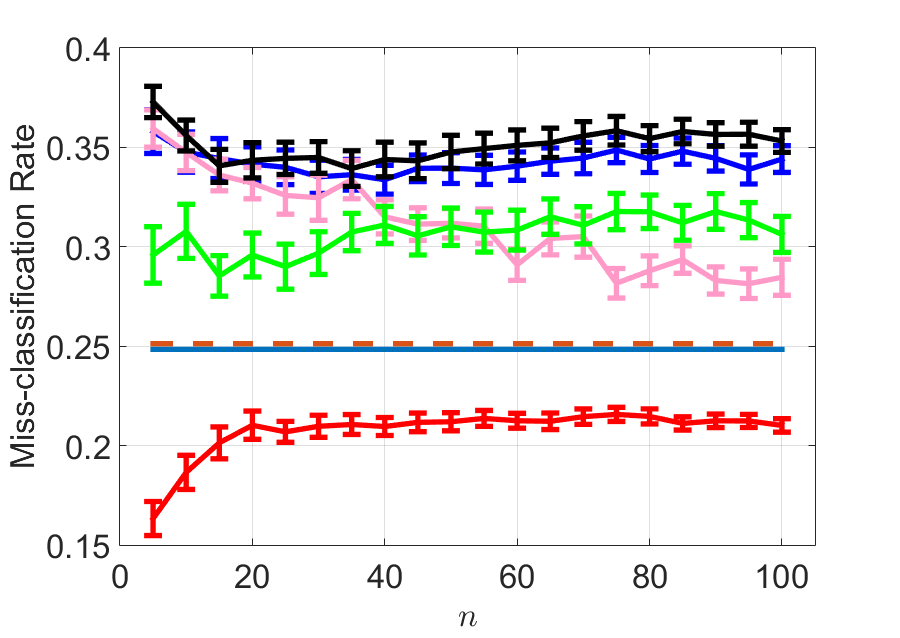}}
	\subfigure[talk.religion.misc]{\includegraphics[width=.3\textwidth]{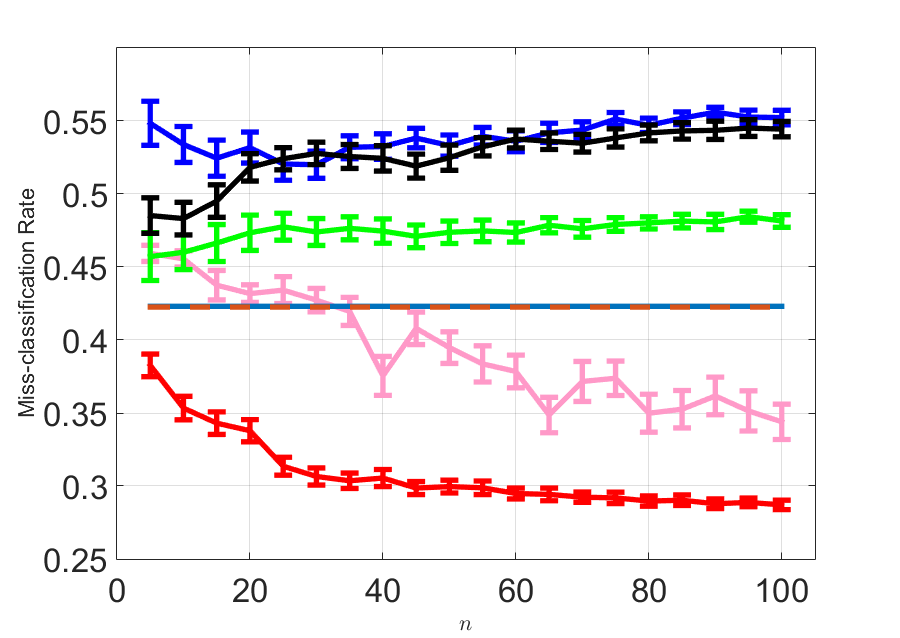}}
	\caption{20 News datasets.}
	\label{fig.20news}
\end{figure*}
\begin{figure*}[t]
	\centering
	\subfigure[kitchen]{\label{fig.kitchen}\includegraphics[width=.3\textwidth]{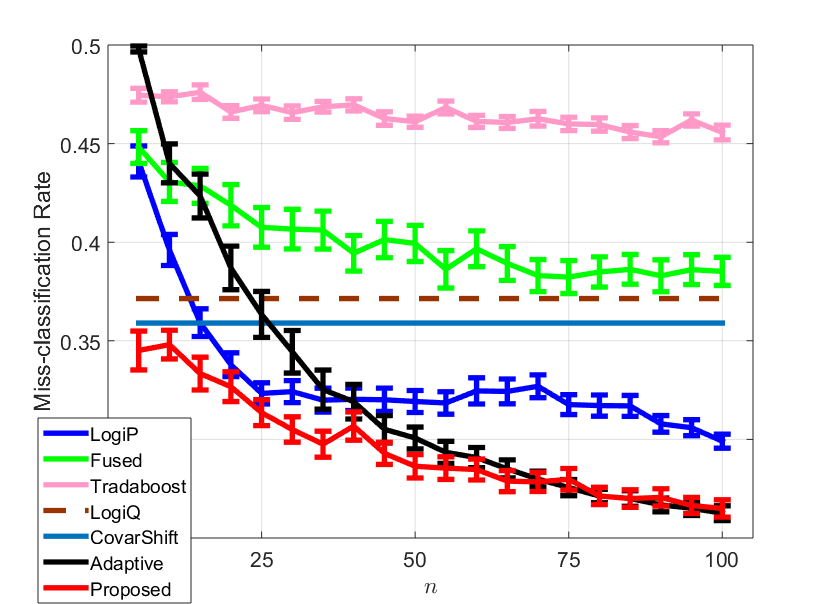}}
	\subfigure[dvd]{\label{fig.dvd}\includegraphics[width=.3\textwidth]{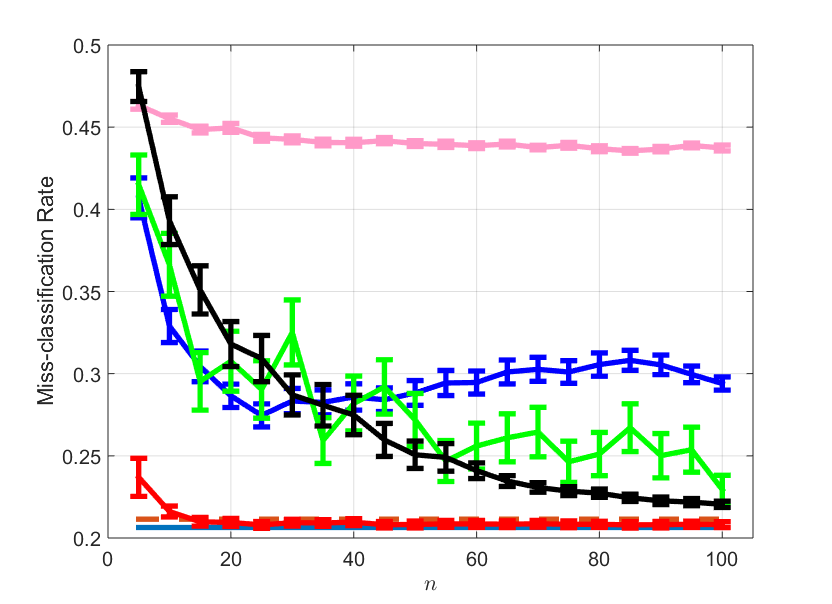}}
	\subfigure[books]{\label{fig.books}\includegraphics[width=.3\textwidth]{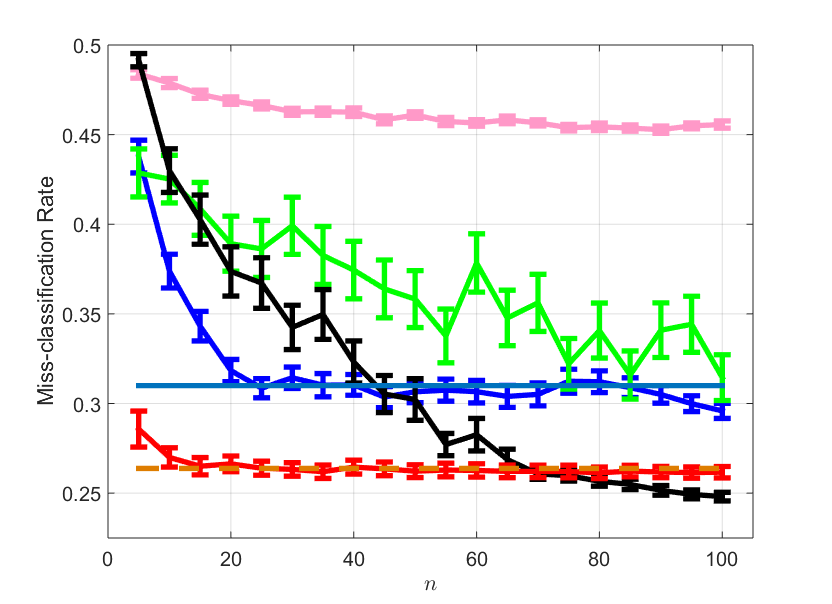}}
	\caption{Amazon sentiment datasets.}
	\label{fig.sentiment}
\end{figure*}
\section{Experiments}
We fix the feature function $\boldf$ as
$
\boldf(\boldx,y) := y\left[ \boldx, 1\right]^\top.
$
It is consistent with our ``simple transfer model'' assumption discussed in Section \ref{sec.model}.
\subsection{Synthetic Experiments}
\label{sec.synt}
	\paragraph{KL convergence} The first experiment uses our trained posterior ratio model to approximate the conditional KL divergence. Since our estimate $\hat{\boldtheta} \rightarrow \boldtheta^*$, we hope to see
%	The estimated parameter $\hat{\boldtheta} \rightarrow \boldtheta^*$ when number of samples increases. If the model is identifiable, then 
$
	\ell(\hat{\boldtheta}; \dataset_P, \dataset_Q) - \mathrm{KL}\left[p\|q\right]\rightarrow 0
$
	as $n, n' \rightarrow \infty$.
	We draw two balanced-classes of samples from two Gaussian distributions with different means for $P$ and $Q$. Specifically, for $y = \{-1,1\}$, we construct $P$ and $Q$ as follows:
	\begin{align*}
	q(x|1) = \mathrm{Normal}(2, 1), & q(x|-1) = \mathrm{Normal}(-2, 1), \\
	p(x|1) = \mathrm{Normal}(1.5, 1), & p(x|-1) = \mathrm{Normal}(-1.5, 1).
	\end{align*}
	%	The illustrations of $q(x|y)$ and $p(x|y)$ are shown in Figure \ref{fig.conv.illus}.
	We draw 5k samples from distribution Q, $n$ samples from $P$, and $k$ is chosen to minimize the error of conditional mean estimation (same below, as it is introduced in the appendix), then train a posterior ratio $g(y,x;\hat{\boldtheta})$.
	By varying $n$ and random sampling, we may create a plot for averaged
	$
	\ell(\hat{\boldtheta}; \dataset_P, \dataset_Q) ,
	$
	with standard error over 25 runs in Figure \ref{fig.conv}. The true conditional KL divergence is plotted alongside as a blue horizontal dash-line. To make comparison, we run the same estimation again with 50k samples from $Q$, and plot in red. 
	
	The result shows, our estimator does converge to the true KL divergence, and the estimation error shrinks as $n \rightarrow \infty$. Increasing $n'$ also help slightly reduce the variance (comparing the blue error bar with the red error bar). However, such improvement is not as significant as increasing $n$. 

	\paragraph{Joint vs. Separated}
	In this experiment, we demonstrate the effect of introducing a ``balancing parameter'' $\gamma$ of the joint optimization method discussed in Section \ref{sec.decom}. We simply reuse the dataset in the previous experiment, and test the averaged negative hold-out likelihood of the approach described in \eqref{eq.joint} and the proposed method using $\dataset_P$ of various sizes. It can be seen that the choice of the parameter $\gamma$ has huge effect on the hold-out likelihood when $n$ is small. However, the proposed method is free from such parameter and can achieve a very low likelihood even when using only 10 samples from $\dataset_P$.

	\paragraph{4-Gaussian}
	The second experiment demonstrates how a simple transfer model helps transfer a non-linear classifier. The dataset $\dataset_Q$ is constructed using mixtures of Gaussian distributions with different means on horizontal axis and two classes of samples are not linearly separable. To create dataset $\dataset_P$, we simply shift their means away from each other on the vertical dimension (See Figure \ref{fig.illus}). We compare the posterior functions learned by kernel logistic regression performed on $\dataset_P$ (Figure \ref{fig.4gau.p}) and $\dataset_Q$ (Figure \ref{fig.4gau.q}) with the proposed transfer learning method (Figure \ref{fig.4gau.proposed}) which is a multiplication of the learned $g(y,\boldx;\hat{\boldtheta})$ and  $q(y|\boldx;\hat{\boldbeta}_q)$.
	
	We set $n = 40$, $n'=5000$. It can be seen from Figure \ref{fig.4gau.p} that although kernel logistic regression has learned the rough decision boundary by using $\dataset_P$ only, it has completely missed the characteristics of the posterior function near the class border due to lack of observations. In contrast, built upon a successfully learned posterior function on dataset $\dataset_Q$ (Figure \ref{fig.4gau.q}), the proposed method successfully transferred the posterior function for the new dataset $\dataset_P$, even though it is equipped only with linear features (Figure \ref{fig.4gau.proposed}). The classification boundary it provides is highly non-linear.

\subsection{Real-world Applications}
\paragraph{20-news}
	Experiments are run on 20-news dataset where articles are grouped into major categories (such as ``sports'') and sub-categories (such as ``sports.basketball''). In this experiment, we adopt ``one versus the others'' scenario: i.e. The task is to predict whether an article is drawn from a sub-category or not. 
	%	The target classification task is one sub-category versus the rest prediction with the help of a source task constructed from samples in the same major-category but different sub-categories.
	We first construct $\dataset_P$ by randomly selecting a few samples from a certain sub-category $T$ and then mix them with equal number of samples from the rest of the categories. $\dataset_Q$ is constructed using abundant random samples from the same major- but different sub-categories and random samples from all the rest categories as negative samples.  
	We adopt PCA and reduce the dimension to just 20.
	
	Figure \ref{fig.20news} summarizes the miss-classification rate of the proposed transfer learning algorithm and all the other methods: \textbf{LogiP} logistic regression on $\dataset_P$,
	\textbf{LogiQ} logistic regression on $\dataset_Q$, \textbf{TrAdaBoost} \cite{dai2007boosting},
	\textbf{Reg} \cite{regularizedMultitaskLearning}, \textbf{CovarShift} \cite{Covariate_Shift_jour,kanamori2009leastsquare} and \textbf{Adaptive} \cite{yang2007cross} over different sub-category $T$ in the ``sci'' and ``talk'' category. The result shows that the proposed method works well in almost all cases, while the comparison methods \textbf{Reg} \textbf{CovarShift} and \textbf{TrAdaBoost}, some times have difficulties in beating the naive base line \textbf{LogiP} and \textbf{LogiQ}. In most cases, \textbf{Adaptive} cannot improve much from \textbf{LogiP}.

\paragraph{Amazon Sentiment}
	The final experiment is conducted on the Amazon sentiment dataset, where the task is to classify the positive or negative sentiment from user's review comments on ``kitchen, electronics, books and dvds''. Since some of the products (such as electronics) are far better reviewed than the others (such as kitchen tools ), it is ideal to transfer a classifier from a well-reviewed product to another one. 
	
	In this experiment, we first sample $\dataset_P$ from one product $T$ and construct dataset $\dataset_Q$ using all samples from all other products. We apply locality preserving projection \cite{he2005neighborhood} to reduce the original dimension from $\sim 148000$ to 30. 
	
	The classification error rate is reported in Figure \ref{fig.sentiment} for $T$ = ``kitchen'', ``dvd'' and ``books''. We omit the $T$ = ``electronics'' since it is noticed that \textbf{logiP} and \textbf{logiQ} has very close performance on this dataset suggesting transfer learning is not helpful.
	
	It can be seen that the proposed method has also achieved low miss-classification rate on all three datasets, even though \textbf{Adapvtive} gradually catches up when $n$ is large enough. Interestingly, Figure \ref{fig.dvd} and \ref{fig.books} show that \textbf{logiQ} can achieve very low error rate, and the proposed method manage to reach similar rates. Even if the benefit of transferring is not clear in these two cases, the proposed method does not seem to bring in extra errors by also considering samples from target dataset $\dataset_P$ which could have been misleading.
%	It can be seen that the proposed method also has achieved low miss-classification rate on all three datasets. Interestingly, Figure \ref{fig.elec} shows that \textbf{logiP} has 0 classification rate, and the proposed method achieves slightly higher but a very close performance. The similar pattern can be found for \textbf{logiQ} in Figure \ref{fig.dvd}. Even if the benefit of transferring is not clear in these two cases, the proposed method does not bring in much extra errors from the use of either source or target domain dataset.
	
\section{Conclusions}
	As modern classifiers get increasingly complicated, the cost of transfer learning become major concern: As in many applications, the transfer should be both quick and accurate. To reduce the modeling complexity, we introduce a composite method: learn a posterior ratio and the source probabilistic classifier separately then combine them together later. As the posterior ratio allows the incremental modeling, features, no matter how complicated, can be ignored as long as they do not participate in the dataset transfer. The posterior ratio is learned via an efficient convex optimization and is proved consistent. Experiments on both artificial and real-world datasets  give promising results.
\bibliographystyle{plain}
\bibliography{main}
	\clearpage
%	\onecolumn 
	\section*{Appendix, Proof for Lemma 1}
	\begin{proof}
		First, we decompose supremum of the approximation error of the empirical objective function:
		\begin{align}
		&\sup_\boldtheta  \left|P_n \log g_{n'}(y,\boldx;\boldtheta) - P \log g(y,\boldx;\boldtheta) \right|\notag\\
		=& \sup_\boldtheta  |\left(P_n \boldtheta^\top \boldf(y,\boldx) - P_n\log N_{n'}\left(\boldtheta;\boldx\right)\right) - \notag\\
		&~~~~~~\left(P \boldtheta^\top \boldf(y,\boldx) - P \log N\left(\boldtheta;\boldx\right)\right) | \notag \\
		\le& \sup_\boldtheta  \left|(P_n - P) \boldtheta^\top \boldf(y,\boldx)\right|  + \notag\\ 
		&\sup_\boldtheta \left|P_n\log N_{n'}\left(\boldtheta;\boldx\right) - P \log N\left(\boldtheta;\boldx\right) \right| \notag\\
		\le& \sup_\boldtheta  \left|(P_n - P) \boldtheta^\top \boldf(y,\boldx) \right|  + \notag\\
		&\sup_\boldtheta | (P_n - P) \log N_{n'}\left(\boldtheta;\boldx\right) + \notag\\
		&~~~~~~P\log N_{n'}\left(\boldtheta;\boldx\right) - P \log N\left(\boldtheta;\boldx\right) | \notag\\
		\le& \sup_\boldtheta \left| (P_n - P) \boldtheta^\top \boldf(y,\boldx) \right| + \label{eq.upperbound.difference}\\
		&\sup_\boldtheta  \left|(P_n- P)\log N_{n'}\left(\boldtheta;\boldx\right) \right| 
		+ \notag\\
		&\sup_\boldtheta \left|P\log N_{n'}\left(\boldtheta;\boldx\right) - P\log N\left(\boldtheta;\boldx\right)\right| \notag
		\end{align}
		The first two terms in \eqref{eq.upperbound.difference} is due to the approximation using samples from $\dataset_P$, while the third term is the \emph{model approximation error} caused by using $k$-NN to approximate $N(\boldx;\boldtheta)$.
		The first two terms are relatively easy to bound. The Uniform Law of Large Numbers (see, e.g. Lemma 2.4 in \cite{uniformLLN}) can be applied to show the first two terms converges to 0 in probability, since i. $\boldtheta$ is compact, ii. both $\boldtheta^\top \boldf(y,\boldx)$ and $\log N_{n'}$ are continuous over $\boldtheta$, iii. both above functions are Lipschitz continuous as we will show later. As to the third term, we first prove for all $\epsilon>0$
		\begin{align*}
			\mathrm{Prob}\left(\sup_\boldtheta P\log N_{n'}\left(\boldtheta;\boldx\right) - P\log N\left(\boldtheta;\boldx\right) > \epsilon \right) \to 0
		\end{align*}
		by using the following inequality:
		$\log a - \log b \le \frac{a-b}{b}$.
		\begin{align}
		&\sup_\boldtheta P~~\log N_{n'}\left(\boldtheta;\boldx\right) - P~~ \log N\left(\boldtheta;\boldx\right) \notag\\
		\le& \sup_\boldtheta P~~ \frac{N_{n'}\left(\boldtheta;\boldx\right) -  N\left(\boldtheta;\boldx\right)}{N\left(\boldtheta;\boldx\right)}\notag\\
		\le& \sup_\boldtheta P~~ \left|\frac{N_{n'}\left(\boldtheta;\boldx\right) -  N\left(\boldtheta;\boldx\right)}{N\left(\boldtheta;\boldx\right)}\right|\notag\\
		\le& M_{\text{max}}\sup_\boldtheta P~~ \left|N_{n'}\left(\boldtheta;\boldx\right) -  
		N\left(\boldtheta;\boldx\right)\right|
		\notag\\
		\le&  R_{\text{max}} M_{\text{max}}\sup_\boldtheta Q~~ \left|N_{n'}\left(\boldtheta;\boldx\right) -  N\left(\boldtheta;\boldx\right)\right| \label{eq.golden.equation}
		\end{align}
		To show the final line converges to 0 with probability one, we use the \emph{Generic Uniform Law of Large Numbers (Generic ULLN)} (see \cite{andrews1992generic} Theorem 1.):
		\begin{mytheorem}[Generic ULLN]
			For a random sequence $\{G_n(\boldtheta), \boldtheta \in \Theta, n\ge 1\}$, if $\Theta$ is a totally bounded metric space, $G_n(\boldtheta)$ is stochastic equicontinous (SE) and $\left|G_n(\boldtheta) \right|\overset{p}{\to} 0, \forall \boldtheta$, then $\sup_{\boldtheta} G_{n}(\boldtheta) \overset{p}{\to} 0$ as $n\to \infty$.
		\end{mytheorem}
		Since by assumption, $\boldtheta$ is bounded. We now verify the rest two conditions of this theorem.
		The universal consistency of $k-$NN has been proved (see \cite{gyorfi2002distribution}, Theorem 23.8, 23.7). Here we restate the results for our conveniences:
		\begin{mytheorem}[Universal consistency of KNN]
			\label{thm.knn.uni.consis}
			Given $Z$ is bounded, assume that for each $\boldx$, the random variable $\|X-\boldx\|$ is absolutely continuous, if $k_{n'} / \log n' \rightarrow \infty $ and $k_{n'} / n' \rightarrow 0 $, $k_{\text{n'}}-$NN estimator is strongly universally consistent, i.e., 
			\begin{align*}
			\lim_{n'\rightarrow \infty} \int  \left|\frac{1}{k}\sum_{j\in \mathcal{N}_{n',k}(\boldx)} z_j- \mathbb{E}\left[Z|X=\boldx\right]\right|^2 d\mu(\boldx)\rightarrow 0 
			\end{align*}
			with probability one for all distributions $(Z,X)$,
			where $\mu(\boldx)$ is the probability measure of $\boldx$.
		\end{mytheorem}
		From Jensen's inequality, we have 
		\begin{align*}
		&\left\{\int  \left|\frac{1}{k}\sum_{j\in \mathcal{N}_{n',k}(\boldx)} z_j- \mathbb{E}\left[Z|X=\boldx\right]\right| d\mu(\boldx))\right\}^2\\
		\le &\int  \left|\frac{1}{k}\sum_{j\in \mathcal{N}_{n',k}(\boldx)} z_j- \mathbb{E}\left[Z|X=\boldx\right]\right|^2 d\mu(\boldx),
		\end{align*}
		and it can be seen that the left hand side also converges to 0 in probability. By using the Continuous Mapping Theorem, we can finally show that 
		$\int  \left|\frac{1}{k}\sum_{j\in \mathcal{N}_{n',k}(\boldx)} z_j- \mathbb{E}\left[Z|X=\boldx\right]\right| d\mu(\boldx)) $ converges to 0 in probability.
		
		We let $Z_\boldtheta = \exp\left(\boldtheta^\top \boldf(Y, X)\right)$ be a new random variable and thus we have samples $\left\{z_{\boldtheta,i}, \boldx_i\right\}_{i=1}^{n'} \sim (Z_\boldtheta,X)$ drawn from distribution $Q$, and 
		\begin{align*}
		&Q ~~ \left|\left(N_{n'}\left(\boldtheta;\boldx\right) -  N\left(\boldtheta;\boldx\right)\right) \right| \\
		= &\int \mu(\boldx) \left|\frac{1}{k}\sum_{j\in \mathcal{N}_{n',k}(\boldx)} z_{\boldtheta,j}- \mathbb{E}\left[Z_{\boldtheta}|X=\boldx\right]\right| d\mu(\boldx).
		\end{align*}
		By applying the Theorem \ref{thm.knn.uni.consis}, we can conclude, such $Q ~~ \left|N_{n'}\left(\boldtheta;\boldx\right) -  N\left(\boldtheta;\boldx\right)\right|$ converges 0 in probability for all distribution $(Z_\boldtheta,X)$ indexed by parameter $\boldtheta$.
		Next, we verify the SE of $Q ~~ \left|N_{n'}\left(\boldtheta;\boldx\right) -  N\left(\boldtheta;\boldx\right)\right|$. Given Assumption \ref{ass.bounded.model}, we have
		\begin{align}
		\label{eq.lip}
		&Q \left|N_{n'}\left(\boldtheta;\boldx\right) -  N\left(\boldtheta;\boldx\right)\right| - Q  \left|N_{n'}\left(\boldtheta;\boldx\right) -  N\left(\boldtheta;\boldx\right)\right| \notag\\
		\le & Q \left|N_{n'}\left(\boldtheta;\boldx\right) -  N_{n'}\left(\boldtheta';\boldx\right) + N\left(\boldtheta';\boldx\right) -  N\left(\boldtheta;\boldx\right)\right|\notag\\
		\le &  2M_\text{max} F_\mathrm{max} \|\boldtheta - \boldtheta'\|_2.
		\end{align}
		The last line is due to Mean-value Theorem:
		\begin{align*}
		&\exp\left(\boldtheta^\top f(y,\boldx)\right) - \exp\left(\boldtheta'^\top \boldf(y,\boldx)\right) \\
		\le& \|\boldtheta - \boldtheta'\| \|\boldf(y,\boldx)\exp(\bar{\boldtheta}^\top \boldf(y,\boldx))\|\\
		\le& \|\boldtheta - \boldtheta'\|  F_\mathrm{max} M_\mathrm{max},
		\end{align*}
		where $\bar{\boldtheta}$ is a vector in-between $\boldtheta$ and $\boldtheta'$ elementwisely.
		
		 In fact, \eqref{eq.lip} shows the function $Q ~~ \left|N_{n'}\left(\boldtheta;\boldx\right) -  N\left(\boldtheta;\boldx\right)\right|$ is \emph{Lipschitz continuous} with respect to $\boldtheta$, and according to Lemma 2 in \cite{andrews1992generic}, it implies SE. Similarly, one can show that $N_{n'}(\boldx;\boldtheta)$ is Lipschitz continuous.
		
		Now we can utilize the property of i. boundedness of $\boldtheta$, ii. SE and iii. universal consistency to conclude that \[\sup_\boldtheta Q~~ \left|N_{n'}\left(\boldtheta;\boldx\right) -  N\left(\boldtheta;\boldx\right)\right| \overset{p}{\to} 0,\]
		and due to \eqref{eq.golden.equation}:
		\[\mathrm{Prob}\left(\sup_\boldtheta P\log N_{n'}\left(\boldtheta;\boldx\right) - P\log N\left(\boldtheta;\boldx\right) \ge \epsilon \right)  {\to} 0. \]
		
		Similarly, one can prove that 
		\begin{align*}
		\mathrm{Prob}\left(\sup_\boldtheta P\log N\left(\boldtheta;\boldx\right) - P\log N_{n'}\left(\boldtheta;\boldx\right) \ge \epsilon\right) {\to} 0. 
		\end{align*}
		 As a consequence, the third term in \eqref{eq.upperbound.difference} $\overset{p}{\to} 0$. 
	\end{proof}
	After obtaining Lemma \ref{lem.convergence}, the rest is similar to the proof of Theorem 9.13 in \cite{WassermanAllStat2010}. Let $M(\boldtheta) := P \log g(y,x;\boldtheta)$ and $M_{n,n'}(\boldtheta) := P_n \log g_{n'}(y,x;\boldtheta)$ and
	\begin{align*}
	M(\boldtheta^*)- & M(\hat{\boldtheta}) = M(\boldtheta^*)- M_{n,n'}(\hat{\boldtheta}) + M_{n,n'}(\hat{\boldtheta}) - M(\hat{\boldtheta})\\
	&\le M(\boldtheta^*)- M_{n,n'}(\boldtheta^*) + M_{n,n'}(\hat{\boldtheta}) - M(\hat{\boldtheta})\\
	&\le M(\boldtheta^*) - M_{n,n'}(\boldtheta^*) + \sup_{\boldtheta} \left|M_{n,n'}(\boldtheta) - M(\boldtheta) \right|.
	\end{align*}
	The last line converges to 0 in probability is proved in Lemma \ref{lem.convergence}. Therefore, we can write: 
	\begin{align*}
	\forall \epsilon >0 , P(M(\boldtheta^*)- M(\hat{\boldtheta}) \ge \epsilon) \rightarrow 0.
	\end{align*}
	Due to Assumption \ref{assum.identi}, for an arbitrary choice of $\epsilon_0>0$, if $\|\boldtheta^*- \hat{\boldtheta}\| \ge \epsilon_0$, there must be a $\epsilon>0$, so that $M(\boldtheta^*)- M(\hat{\boldtheta}) > \epsilon$. Therefore, we conclude
	\begin{align*}
	\forall \epsilon_0 >0 , P(\|\boldtheta^*- \hat{\boldtheta} \| \ge \epsilon_0) \le P(M(\boldtheta^*)- M(\hat{\boldtheta}) \ge \epsilon) \rightarrow 0.
	\end{align*}
	Also, $M_n(\hat{\boldtheta}) - M(\boldtheta^*) = M_n(\hat{\boldtheta}) - M(\hat{\boldtheta}) + M(\hat{\boldtheta}) - M(\boldtheta^*)$. Due to Lemma \ref{lem.convergence}, it converges to 0 in probability. Therefore, we have $\ell(\hat{\boldtheta}; \dataset_P, \dataset_Q) \overset{p}{\to} \mathrm{KL}\left[p\|q\right]$.
	\section*{Tuning Parameters in Posterior Ratio Estimation}
	\textbf{$k$ in $k$-NN}: As it is mentioned in Section \ref{sec.synt}, $k$ is tuned via 5-fold cross validation, and is based on the testing criterion:
	\begin{align}
	\label{eq.holdout}
\mathrm{MSE} = 
\frac{1}{|\dataset_{\mathrm{HO}}|}\sum_{j\in\dataset_{\mathrm{HO}}}\left( Z_q^{(j)} - 
	\frac{1}{k}\sum_{jj \in \mathcal{N}(\xqj)} Z_q^{(jj)}\right)^2,
	\end{align} where $\dataset_{\mathrm{HO}}$ is a holdout dataset and $Z_q^{(i)} = \exp\left(\boldtheta^\top \boldf(y_q^{(i)},\boldx_q^{(i)})\right)$. However, such value depends on $\boldtheta$ and it changes every iteration during the gradient decent. Instead of tuning $k$ after each iteration, we follow a simple heuristics: 1) Fix $k$ and run gradient descent. 2) choose a suitable $k$ that minimizes \eqref{eq.holdout}. 1) and 2) are repeatedly carried out until converge. Such heuristics have very good performance in experiments.
\end{document}

%% file: notation_song.tex
\newtheorem{mytheorem}{Theorem}
\newtheorem{mylemma}{Lemma}
\newcommand{\argmax}{\mathop{\rm argmax}\limits}
\newcommand{\argmin}{\mathop{\rm argmin}\limits}

%Highlight
%\newcommand{\hilight}[1]{\colorbox{yellow}{#1}}
%Disable Highlight
\newcommand{\hilight}[1]{{#1}}
\newcommand{\highlight}[2][yellow]{\mathchoice%
	{\colorbox{#1}{$\displaystyle#2$}}%
	{\colorbox{#1}{$\textstyle#2$}}%
	{\colorbox{#1}{$\scriptstyle#2$}}%
	{\colorbox{#1}{$\scriptscriptstyle#2$}}}%

\newcommand{\unorm}[1]{\|#1\|}
\newcommand{\unorms}[1]{\unorm{#1}^2}
\newcommand{\calX}{{\mathcal{X}}}
\newcommand{\calY}{{\mathcal{Y}}}
\newcommand{\gonetwo}{{g_{12}}}
\newcommand{\gonetwobar}{{\bar{g}_{12}}}
\newcommand{\calB}{{\mathcal{B}}}
\newcommand{\calBitwo}{{\mathcal{B}_{i,2}}}
\newcommand{\calBione}{{\mathcal{B}_{i,1}}}
\newcommand{\boldtheta}{{\boldsymbol{\theta}}}
\newcommand{\boldphi}{{\boldsymbol{\phi}}}
\newcommand{\bolddelta}{{\boldsymbol{\delta}}}
\newcommand{\boldthetaP}{{\boldsymbol{\theta}}^{(p)}}
\newcommand{\boldthetaQ}{{\boldsymbol{\theta}}^{(q)}}
\newcommand{\boldthetaPtop}{{\boldsymbol{\theta}}^{(p)\top}}
\newcommand{\factorp}{{\phi}^P}
\newcommand{\factorq}{{\phi}^Q}
\newcommand{\boldalpha}{{\boldsymbol{\alpha}}}
\newcommand{\boldHh}{{\widehat{\boldH}}}
\newcommand{\boldeta}{{\boldsymbol{\eta}}}
\newcommand{\boldH}{{\boldsymbol{H}}}
\newcommand{\boldA}{{\boldsymbol{A}}}
\newcommand{\boldS}{{\boldsymbol{S}}}
\newcommand{\boldK}{{\boldsymbol{K}}}
\newcommand{\boldJ}{{\boldsymbol{J}}}
\newcommand{\boldT}{{\boldsymbol{T}}}
\newcommand{\boldTheta}{{\boldsymbol{\Theta}}}
\newcommand{\boldf}{{\boldsymbol{f}}}
\newcommand{\boldu}{{\boldsymbol{u}}}
\newcommand{\boldm}{{\boldsymbol{m}}}
\newcommand{\boldone}{{\boldsymbol{1}}}
\newcommand{\boldxi}{{\boldsymbol{\xi}}}
\newcommand{\boldv}{{\boldsymbol{v}}}
\newcommand{\boldk}{{\boldsymbol{k}}}
\newcommand{\boldb}{{\boldsymbol{b}}}
\newcommand{\boldbeta}{{\boldsymbol{\beta}}}
\newcommand{\boldDelta}{{\boldsymbol{\Delta}}}
\newcommand{\nnu}{\nsample}
\newcommand{\nsample}{n}
\newcommand{\subsetr}{\boldsymbol{r}}
\newcommand{\boldthetah}{{\widehat{\boldtheta}}}
\newcommand{\mathbbR}{\mathbb{R}}
\newcommand{\KL}{\mathrm{KL}}
\newcommand{\numparams}{n}
\newcommand{\boldhh}{{\widehat{\boldh}}}
\newcommand{\boldh}{{\boldsymbol{h}}}
\newcommand{\Hh}{{\widehat{H}}}
\newcommand{\boldxnu}{\boldY}
\newcommand{\boldx}{{\boldsymbol{x}}}
\newcommand{\boldxp}{{\boldsymbol{x}}_{p}}
\newcommand{\boldxq}{{\boldsymbol{x}}_{q}}
\newcommand{\boldz}{{\boldsymbol{z}}}
\newcommand{\boldg}{{\boldsymbol{g}}}
\newcommand{\boldw}{{\boldsymbol{w}}}
\newcommand{\boldr}{{\boldsymbol{r}}}
\newcommand{\boldQ}{{\boldsymbol{Q}}}
\newcommand{\boldF}{{\boldsymbol{F}}}
\newcommand{\boldzero}{{\boldsymbol{0}}}
\newcommand{\thetahat}{{\hat{\boldsymbol{\theta}}}}
\newcommand{\thetaShat}{{\hat{\boldsymbol{\theta}}_S}}
\newcommand{\thetaSchat}{{\hat{\boldsymbol{\theta}}_{S^c}}}
\newcommand{\zhat}{{\hat{\boldsymbol{z}}}}
\newcommand{\zSchat}{{\hat{\boldsymbol{z}}_{S^c}}}
\newcommand{\zShat}{{\hat{\boldsymbol{z}}_{S}}}
\newcommand{\nde}{\nsample'}
\newcommand{\boldxde}{\boldY'}
\newcommand{\boldX}{{\boldsymbol{X}}}
\newcommand{\boldY}{{\boldsymbol{Y}}}
\newcommand{\boldy}{{\boldsymbol{y}}}
\newcommand{\boldt}{{\boldsymbol{t}}}
\newcommand{\boldYnu}{{\boldsymbol{Y}}}
\newcommand{\boldYde}{{\boldsymbol{Y}}}
\newcommand{\boldpsi}{{\boldsymbol{\psi}}}
\newcommand{\hh}{{\widehat{h}}}
\newcommand{\boldI}{{\boldsymbol{I}}}
\newcommand{\PE}{{\widehat{PE}}}
\newcommand{\ratioh}{\widehat{\ratiosymbol}}
\newcommand{\ratiosymbol}{r}
\newcommand{\ratiomodel}{g}
\newcommand{\thetah}{{\widehat{\theta}}}
\newcommand{\mathbbE}{\mathbb{E}}
\newcommand{\pnu}{p_\mathrm{te}}
\newcommand{\pde}{p_\mathrm{rf}}
\newcommand{\refsection}{\boldS_\mathrm{rf}}
\newcommand{\tesection}{\boldS_\mathrm{te}}
\newcommand{\refY}{\boldY_\mathrm{rf}}
\newcommand{\teY}{\boldY_\mathrm{te}}
\newcommand{\nseg}{n}
\newcommand{\distP}{P}
\newcommand{\distQ}{Q}
\newcommand{\iid}{\stackrel{\mathrm{i.i.d.}}{\sim}}
\newcommand{\dx}{\mathrm{d}\boldx}
\newcommand{\dy}{\mathrm{d}\boldy}
\newcommand{\ypi}{{y_p^{(i)}}}
\newcommand{\xpi}{{\boldx_p^{(i)}}}
\newcommand{\yqj}{{y_q^{(j)}}}
\newcommand{\xqj}{{\boldx_q^{(j)}}}

\newcommand{\gxeta}{g(\boldx;\boldeta)}
\newcommand{\Zeta}{Z(\boldeta)}
\newcommand{\Zetahat}{\hat{Z}(\boldeta)}

\def\ratio{r}
\def\relratio{{\ratio}_{\alpha}}

\def\ci{\perp\!\!\!\perp} % from Wikipedia
\newcommand\independent{\protect\mathpalette{\protect\independenT}{\perp}} % symbols-a4, p.106
\def\independenT#1#2{\mathrel{\rlap{$#1#2$}\mkern2mu{#1#2}}} 
\newcommand*\xor{\mathbin{\oplus}}

\newcommand{\vertiii}[1]{{\left\vert\kern-0.25ex\left\vert\kern-0.25ex\left\vert #1 
		\right\vert\kern-0.25ex\right\vert\kern-0.25ex\right\vert}}

\newcommand{\dataset}{{\cal D}}
\newcommand{\fracpartial}[2]{\frac{\partial #1}{\partial  #2}}
\newtheorem{lem}{Lemma}
\newtheorem{prop}{Proposition}
\newtheorem{them}{Theorem}
\newtheorem{assum}{Assumption}
\newtheorem{corol}{Corollary}
\newtheorem{mydefinition}{Definition}